\documentclass[english,10pt,journal]{IEEEtran}
\usepackage{color,amsmath,amssymb,graphicx,amsthm,epsfig,mathrsfs,cite,bm,enumerate}
\usepackage{array}
\usepackage{verbatim}
\usepackage{mathtools,multirow}

\usepackage{color,soul}
\usepackage{float}
\makeatletter
\newcommand*{\rom}[1]{\expandafter\@slowromancap\romannumeral #1@}
\makeatother
\def\supp{\operatorname{supp}}

\def\P{\mathbb{P}}

\newcommand{\upperRomannumeral}[1]{\uppercase\expandafter{\romannumeral#1}}

\newtheorem{lemma}{Lemma}{}
  \newtheorem{thm}{Theorem}

\theoremstyle{remark} \newtheorem{remark}{Remark}
\usepackage{amssymb}
\usepackage{float}

\usepackage{algorithm}
\usepackage{algpseudocode}
\usepackage{lipsum}
\usepackage{enumerate}

\newcommand{\squeezeup}{\vspace{-2.5mm}}

\title{High SNR Consistent Compressive Sensing}

\ifCLASSOPTIONtwocolumn
\author{Sreejith Kallummil\hspace{0cm},   Sheetal Kalyani  \\
 Department of Electrical Engineering \\
 Indian Institute of  Technology Madras,\\
 Chennai, India 600036  \\
 \{ee12d032,skalyani\}@ee.iitm.ac.in}
\else
\author{{Sreejith Kallummil\hspace{0cm},   Sheetal Kalyani  \\
 Department of Electrical Engineering \\
 Indian Institute of  Technology Madras,\\
 Chennai, India 600036  \\
 \{ee12d032,skalyani\}@ee.iitm.ac.in
 }}
\fi
\begin{document}
\maketitle
\begin{abstract}
High signal to noise ratio (SNR) consistency of model selection criteria in linear regression models has attracted a lot of attention recently. However, most of the existing literature on high SNR consistency deals with model order selection. Further, the limited literature available on the high SNR consistency of subset selection procedures (SSPs) is applicable  to linear regression  with full rank measurement matrices only. Hence, the performance of SSPs used  in underdetermined linear models (a.k.a compressive sensing (CS) algorithms) at high SNR is largely unknown. This paper fills this gap by deriving necessary and sufficient conditions for the high SNR consistency of popular CS algorithms like $l_0$-minimization, basis pursuit de-noising or LASSO, orthogonal matching pursuit and  Dantzig  selector. Necessary conditions analytically establish the high SNR inconsistency of  CS algorithms when used with the tuning parameters discussed in literature. Novel tuning parameters with  SNR adaptations are developed using the sufficient conditions and the choice of SNR adaptations are discussed analytically using convergence rate analysis. CS algorithms with the proposed tuning parameters are numerically shown to be high SNR consistent and  outperform existing tuning parameters in the moderate to high SNR regime.
\end{abstract}
\section{Introduction}
 Subset selection or variable selection in linear regression models  is the identification of the support of regression vector $\beta$, i.e., $\mathcal{I}=supp(\beta)=\{j:\beta_j\neq 0\}$ in the regression model ${\bf y}={\bf X}{\beta}+{\bf w}$. Here, ${\bf X}\in \mathbb{R}^{n\times p}$ is a known design matrix with unit $l_2$ norm columns,  ${\bf y} \in \mathbb{R}^n$ is the observed vector and ${\bf w}\sim \mathcal{N}({\bf 0}_n,\sigma^2{\bf I}_n)$ is the additive white Gaussian noise with known variance $\sigma^2$. Let 
$k^*$ denotes the number of non-zero entries in $\beta$. In this paper, we consider subset selection in underdetermined  linear models, i.e., ${\bf X}$ with more columns than rows $(n\leq p)$. This problem studied under the compressive sensing (CS) paradigm is of fundamental importance  in statistical signal processing, machine learning etc.  Many compressive sensing (CS) algorithms with varying performance complexity trade-offs and optimality conditions are available\cite{eldar2012compressed,candes2007dantzig,tropp2004greed,tropp2006just,wipf2004sparse,masood2013sparse} for this purpose. The performance of these CS algorithms are evaluated either in terms of mean square error (MSE) between $\beta$ and the estimate $\hat{\beta}$ returned by the CS algorithm\cite{ben2010coherence} or the correctness with which the estimated support $\hat{\mathcal{I}}=supp(\hat{\beta})$ matches the true support $\mathcal{I}$ \cite{spl}. In this paper, we evaluate CS algorithms in terms of the  probability of  support recovery error  defined by $PE=\mathbb{P}(\hat{\mathcal{I}}\neq \mathcal{I})$. 
 
 Traditionally, PE is evaluated in the large sample regime, i.e., $n\rightarrow \infty$ or $( n,p)\rightarrow \infty$\cite{zhao2006model}. In their landmark paper \cite{ding2011inconsistency}, Ding and Kay  demonstrated that subset selection procedures
 (SSPs) in overdetermined linear models that are large sample consistent (i.e., $PE \rightarrow 0 \ \text{as} \ n \rightarrow \infty$) often performs poorly in a finite $n$ and  high signal to noise ratio (SNR) (i.e., small $\sigma^2$) regime. This result generated great interest in the signal processing community on the behaviour of SSPs as $\sigma^2 \rightarrow 0$. Formally, a SSP is said to  be high SNR consistent if its' $PE\rightarrow 0$ as $\sigma^2 \rightarrow 0$. In this paper, we discuss the high SNR consistency of popular CS algorithms that are used for subset selection in underdetermined linear  models. After presenting the mathematical notations,  we elaborate on the existing literature on high SNR consistency and CS algorithms.  
 \subsection{Notations used in this paper.}
   $col({\bf X})$ the column space of ${\bf X}$. ${\bf X}^T$ is the transpose and ${\bf X}^{\dagger}=({\bf X}^T{\bf X})^{-1}{\bf X}^T$ is the  Moore-Penrose pseudo inverse of ${\bf X}$ (if ${\bf X}$ has full column rank). ${\bf P}_{\bf X}={\bf X}{\bf X}^{\dagger}$ is the projection matrix onto $col({\bf X})$. ${\bf I}_n$ represents an $n\times n$ identity matrix and  ${\bf 0}_n$ represents an $n\times 1$ zero vector. ${\bf X}_{\mathcal{J}}$ denotes the sub-matrix of ${\bf X}$ formed using  the columns indexed by $\mathcal{J}$. ${\bf X}_{i,j}$ is the $[i,j]^{th}$ entry of ${\bf X}$. If ${\bf X}$ is clear from the context, we use the shorthand ${\bf P}_{\mathcal{J}}$ for ${\bf P}_{{\bf X}_{\mathcal{J}}}$. ${\bf a}_{\mathcal{J}}$ or ${\bf a}(\mathcal{J})$ denotes the  entries of ${\bf a}$ indexed by $\mathcal{J}$.  $\mathcal{N} ({\bf u},{\bf C})$ is a Gaussian vector with mean ${\bf u}$ and covariance ${\bf C}$. $\chi^2_j$ is a central chi square distribution with $j$ degrees of freedom (d.o.f) and $\chi^2_j(\lambda)$ is a non central chi square distribution with $j$ d.o.f  and non-centrality $\lambda$.  ${\bf a}\sim{\bf b}$ implies that ${\bf a}$ and ${\bf b}$ are identically distributed.  $|()|$ denotes the absolute value for scalar  arguments and cardinality for set arguments.   $\|{\bf a}\|_q=(\sum\limits_{j}|{\bf a}_j|^q)^{\frac{1}{q}}$  for $1\leq q< \infty$ is the $l_q$ norm, $\|{\bf a}\|_{\infty}=\underset{j}{\max}\ |{\bf a}_j|$ is the $l_{\infty}$ norm and $\|{\bf a}\|_0=|supp({\bf a})|$ is the $l_0$ quasi norm of ${\bf a}$ respectively. ${\bf a}$ is called $k^*$-sparse iff $\|{\bf a}\|_0=k^*$. $\|{\bf A}\|_{m,l}=\underset{\|{\bf x}\|_m=1}{\max}{\|{\bf Ax}\|_l} $  is the $(m,l)^{th}$ matrix norm. $[p]$ denotes the set $\{1,\dotsc,p\}$. For any two index sets $\mathcal{J}_1$ and $\mathcal{J}_2$, the set difference  $\mathcal{J}_1/\mathcal{J}_2=\{j \in \mathcal{J}_1: j\notin  \mathcal{J}_2\}$.  $f(n)=o(g(n))$ iff $\underset{n \rightarrow \infty}{\lim}\dfrac{f(n)}{g(n)}=0$.
 \subsection{Prior literature on high SNR consistency}
  Most of the existing literature related to high SNR consistency including the seminal work by Ding and Kay \cite{ding2011inconsistency} are related to the model order selection (MOS) problem. MOS is a subset selection problem where $\mathcal{I}$ is restricted to the form $\mathcal{I}=[k^*]$.  Another interesting problem related to MOS is the estimation of smallest $\tilde{k}$ such that ${\beta}$ satisfies $\beta_j=0,\forall j>\tilde{k}$ and $\beta_j$ can be zero or non-zero for $j< \tilde{k}$.  In both these cases, the statistician is required to the estimate the model order $k^*$ or $\tilde{k}$. A  number of MOS criteria like exponentially embedded family (EEF)\cite{kay2005exponentially}, normalised maximum likelihood based minimum description length (NMDL)\cite{rissanen2000mdl}, g-prior based MDL (g-MDL)\cite{hansen2001model}, forms of Bayesian Information criteria (BIC)\cite{stoica2012proper,bayesian}, sequentially normalised least squares (SNLS)\cite{SNLS} etc. are proved to high SNR consistent\cite{tsp,ding2011inconsistency,schmidt2012consistency,SNLShighSNR}. All these MOS criteria can  be formulated as the minimization of a penalised log likelihood 
  \begin{equation}\label{pll_mos}
  PLL(k)=\|({\bf I}_n-{\bf P}_{\mathcal{J}_k}){\bf y}\|_2^2+h(k,\sigma^2)\sigma^2
  \end{equation}
   over the collection of subsets $\{\mathcal{J}_k\}_{k=1}^{p}$, where $\mathcal{J}_k=[k]$ and $h(k,\sigma^2)$ is a penalty function. Necessary and sufficient conditions (NSCs) for a MOS criterion to be high SNR consistent is derived  in  \cite{tsp}.  Applying  MOS criteria to the general subset selection problem  where $\mathcal{I}$ can be any subset of $[p]$ involves the minimization of $PLL(\mathcal{J})$  over the entire $2^p$ subsets $\mathcal{J} \subseteq [p]$. This approach though theoretically optimal is computationally intractable.  Consequently a number of suboptimal but low complexity SSPs   are developed.  To the best of our knowledge, only two  SSPs, both of which are based on the least squares (LS) estimate of $\beta$ (i.e., $\hat{\beta}_{LS}={\bf X}^{\dagger}{\bf y}$) are known  to be high SNR consistent\cite{tsp,spl}.
 \subsection {Contributions of this paper}
 The existing literature on high SNR consistency in linear regression   is  applicable only to regression models with full column rank design matrices. Hence, existing literature is not applicable to  underdetermined linear models, i.e., ${\bf X}$ with $n<p$.  Identifying the true support $\mathcal{I}$ in an underdetermined linear model is an  ill-posed problem unless certain structures are imposed on the regression vector $\beta$ and design matrix ${\bf X}$. Throughout this paper, we assume that the regression vector $\beta$ is sparse, i.e., $k^*=|\mathcal{I}|\ll p$ and $k^*<n$. The structure imposed on ${\bf X}$ depends on the particular CS algorithm used.  
 
  This paper makes the following contributions to CS literature from the viewpoint  of high SNR consistency. We first derive NSCs on the tuning parameter $\Gamma_0$ such that the support estimate  $\hat{\mathcal{I}}=supp(\hat{\beta})$ delivered by
  \begin{align*}
  (\text{$l_0$-penalty}) : \hat{\beta}=\underset{{\bf b} \in \mathbb{R}^p}{\arg\min} \|{\bf y}-{\bf X}{\bf b}\|_2^2+\Gamma_0 \sigma^2 \|{\bf b}\|_0,
  \end{align*}
 is high SNR consistent.   It should be noted that optimization problem in \textit{$l_0$-penalty} is NP-hard \cite{foucart2013mathematical}. Hence, a number of suboptimal techniques broadly belonging to two classes, convex relaxation (CR) \cite{tropp2006just,candes2007dantzig} and greedy algorithms \cite{tropp2004greed, masood2013sparse} are developed in literature. We mainly consider  two CR techniques in this paper, \textit{viz}.,
 \begin{align*}
 (\text{$l_1$-penalty}) : \hat{\beta}=\underset{{\bf b} \in \mathbb{R}^p}{\arg\min} \dfrac{1}{2}\|{\bf y}-{\bf X}{\bf b}\|_2^2+\sigma \Gamma_1 \|{\bf b}\|_1 \ \ \text{and}
 \end{align*}
 \begin{align*}
 (\text{$l_1$-error}) : \hat{\beta}=\underset{{\bf b} \in \mathbb{R}^p}{\arg\min} \|{\bf b}\|_1 \ \text{subject to} \ \|{\bf y}-{\bf X}{\bf b}\|_2\leq \sigma\Gamma_2. 
 \end{align*}
 $l_1$-penalty and $l_1$-error are also known as basis pursuit de-noising (BPDN) or least absolute shrinkage and selection operator (LASSO).
 We derive NSCs on $\Gamma_1$, $\Gamma_2$  such that $l_1$-penalty and $l_1$-error  are high SNR consistent. We also derive NSCs on the hyper parameter $\Gamma_3$ of the popular CR technique Dantzig selector \cite{candes2007dantzig} given by
 \begin{align*}
 (\text{DS}) : \hat{\beta}=\underset{{\bf b} \in \mathbb{R}^p}{\arg\min} \|{\bf b}\|_1 \ \text{subject to} \ \|{\bf X}^T({\bf y}-{\bf X}{\bf b})\|_\infty\leq \sigma\Gamma_3
 \end{align*} 
 for the special case of\footnote{In this article we consider a popular formulation of CS algorithms where the tuning parameters are explicitly scaled by $\sigma$ or $\sigma^2$. Quite often  $\sigma$ or $\sigma^2$  is included in the tuning parameter itself. For example, $l_0$-penalty may be written as $\hat{\beta}=\underset{{\bf b} \in \mathbb{R}^p}{\arg\min} \|{\bf y}-{\bf X}{\bf b}\|_2^2+\lambda_0 \|{\bf b}\|_0$. Using the relation $\lambda_0=\sigma^2\Gamma_0$, the NSCs  derived in terms of $\Gamma_0$ can be easily restated in terms of $\lambda_0$ also. } orthonormal ${\bf X}$.
 Orthogonal matching pursuit (OMP) \cite{tropp2004greed,OMP_wang,wang2012recovery,cai2011orthogonal,EEFPDF} is a popular greedy algorithm with sound performance guarantees and low computational complexity in comparison with CR based SSPs. OMP is characterized by its' stopping condition (SC). We also derive high SNR consistent SCs for OMP.
 
 Necessary conditions derived for $l_0$-penalty, $l_1$-penalty, $l_1$-error and DS analytically establish the high SNR inconsistency of these schemes with the  values of $\{\Gamma_k\}_{k=0}^3$ discussed in literature. High SNR inconsistency of  OMP with popular  SCs is numerically established. These inconsistencies are due to the absence of SNR adaptations in the tuning parameters. The sufficient conditions delivers a range of SNR adaptations for tuning parameters that will result in high SNR consistency. To compare various SNR adaptations, we  derived simple bounds on the  convergence rates of $l_1$-penalty.  Extensive numerical simulations conducted on various subset selection scenarios demonstrate the potential of some of these SNR adaptations to significantly outperform existing tuning parameters in the moderate to high SNR regime. In addition to being a topic of  theoretical importance, high SNR consistency of CS algorithms have tremendous practical value. A number of applications such as multi user detection\cite{multiuserCS}, on-off random access\cite{fletcher2009off}, CS based single snapshot direction of arrival   \cite{gerstoft2015multiple} etc. demands  support recovery with very low values of PE in the moderate to high SNR regime.  The high SNR consistent tuning parameters derived in this article can be applied directly for such applications in the moderate to high SNR regime.
\subsection{Organization of paper}
Section \rom{2} gives mathematical preliminaries. Section \rom{3} discuss the high SNR consistency of $l_0$-penalty, Section \rom{4} discuss the consistency of CR techniques and Section \rom{5} discuss the consistency of OMP. Section \rom{6} validates the analytical results through numerical simulations.
\section{Mathematical preliminaries}
In this section, we present a brief overview of mathematical concepts from CS and probability theory used in this article. 
\subsection{ Qualifiers for design matrix ${\bf X}$.}
 When $n<p$, the  linear equation ${\bf y}={\bf X}{\beta}$ has infinitely many possible solutions. Hence the support recovery problem is ill-posed even in the noiseless case. To uniquely recover the $k^*$-sparse vector $\beta$, the measurement matrix ${\bf X}$ has to satisfy certain well known regularity conditions. 

{\bf Definition 1:} The spark of a matrix ${\bf X}$ $\left(spark({\bf X})\right)$ is the smallest number of columns in ${\bf X}$ that are linearly dependent. 

Consider the following the optimization problem.
\begin{equation}\label{maxsparse}
\hat{\beta}=\underset{{\bf b} \in {\bf R}^p}{\arg\min} 
\|{\bf b}\|_0, \ \text{subject to}\ \ {\bf y}={\bf X}{\bf b}. 
\end{equation} 
In words $\hat{\beta }$ is the sparsest vector that solves the linear equation ${\bf y}={\bf Xb}$. The following lemma relates the unique recovery of sparse vectors with  $spark({\bf X})$ in the absence of noise.
\begin{lemma}\label{Sparklemma}
 To uniquely recover all $k^*$-sparse vectors $\beta$ using (1), it is necessary and sufficient that $spark({\bf X})>2k^*$ \cite{eldar2012compressed}.
\end{lemma}
The optimization problem (\ref{maxsparse}) cannot be solved in polynomial time. For polynomial complexity CS algorithms like DS, $l_1$-penalty, $l_1$-error, OMP etc. $spark({\bf X})>2k^*$ is not  sufficient to guarantee unique recovery even in the noiseless case. A plethora of sufficient conditions including restricted isometry property (RIP)\cite{eldar2012compressed,OMP_wang}, mutual incoherence condition (MIC)\cite{tropp2006just,cai2011orthogonal}, exact recovery condition (ERC)\cite{tropp2004greed,tropp2006just} etc. are discussed in the literature.  The high SNR analysis of CR techniques and OMP in this article uses ERC and MIC which are defined next. 

{\bf Definition 2:-} A matrix ${\bf X}$  and a vector $\beta$ with support $\mathcal{I}$ is said to be satisfying ERC if the exact recovery coefficient $erc({\bf X},\mathcal{I})=\underset{j \notin \mathcal{I}}{\max}\|{\bf X}_{\mathcal{I}}^{\dagger}{\bf X}_j\|_1$ satisfies $erc({\bf X},\mathcal{I})<1$.

It is known that ERC is a sufficient and worst case necessary condition for accurately recovering  $\mathcal{I}$ from ${\bf y}={\bf X}\beta$ using OMP and the basis pursuit (BP) algorithm that solves 
\begin{equation}\label{l1_noiseless}
\hat{\beta}= \underset{{\bf b} \in \mathbb{R}^p}{\arg\min}\|{\bf b}\|_1,\ \text{subject to} \ {\bf y}={\bf X}{\bf b}
\end{equation}
in the noiseless case\cite{tropp2004greed,tropp2006just}.  ERC is also used   to study the performance of $l_1$-penalty, $l_1$-error and OMP in noisy data\cite{tropp2006just,cai2011orthogonal}. Since the ERC assumption involves the unknown support $\mathcal{I}$, it is impossible to check ERC in practice. Likewise, verifying the spark assumption is computationally intractable. Hence, the MIC, an assumption which can be easily verified is popular in CS literature\cite{cai2011orthogonal}. 

{\bf Definition 3:-}  A $k^*$-sparse vector $\beta$ satisfies MIC, iff the mutual coherence $\mu_{\bf X}=\underset{i\neq j}{\max}|{\bf X}_{i}^T{\bf X}_j|$ satisfies $\mu_{\bf X}<\dfrac{1}{2k^*-1}$. 

If $\mu_{\bf X}<\frac{1}{2k^*-1}$, then ERC is satisfied for all $k^*$-sparse vector $\beta$, i.e.,  $erc({\bf X},\mathcal{I})<1$\cite{tropp2004greed}.  Likewise,  MIC guarantees that $spark({\bf X})>2k^*$\cite{tropp2004greed}. Since, MIC  implies both ERC and spark assumption, the analysis conducted based on ERC and spark are automatically applicable to problems satisfying MIC. 
\begin{remark} The number of measurements $n$ is an important factor in deciding the properties of ${\bf X}$ like spark, $\mu_{\bf X}$ etc.  In this paper, we will not explicitly quantify $n$, however by stating conditions on  $spark({\bf X})$, $\mu_{\bf X}$, ERC etc.  we implicitly assume that $n$ is sufficiently large enough to satisfy these conditions. 
\end{remark}
\subsection{ Standard Convergence concepts [Chapter 4,\cite{chung2001course}].}
  A collection of random variables (R.Vs) $X_{\sigma^2}$ converges in probability (C.I.P) to a R.V $Y$, i.e., $X_{\sigma^2}\overset{P}{\rightarrow}Y$ as $\sigma^2\rightarrow 0$ iff $\forall\epsilon>0$,   $\underset{\sigma^2\rightarrow 0}{\lim}\P(|X_{\sigma^2}-Y|>\epsilon)=0$.  A R.V X is B.I.P iff it is finite almost everywhere, i.e., for any $\epsilon>0$, $\exists R_{\epsilon}<\infty$ such that $\P(| X|>R_{\epsilon})<\epsilon$. For an event $A$,  $\underset{\sigma^2\rightarrow 0}{\lim}\P(A)=0$ iff for each $\epsilon>0$, $\exists$  $\sigma^2_*(\epsilon)>0$ such that $\P(A)\leq \epsilon$, $\forall \sigma^2<\sigma^2_*(\epsilon)$.  Next we describe the relationship between projection matrices and $\chi^2$ R.Vs\cite{tsp}. 
 \begin{lemma}\label{chisquarelemma}
 Let ${\bf P}$ be an arbitrary $n\times n$ projection matrix with  rank $j$. Then for any ${\bf z}\sim\mathcal{N}({\bf u},\sigma^2{\bf I}_n)$, $\dfrac{\|{\bf P}{\bf z}\|_2^2}{\sigma^2}\sim\chi^2_j(\dfrac{\|{\bf P}{\bf u}\|_2^2}{\sigma^2})$ and $\dfrac{\|({\bf I}_n-{\bf P}){\bf z}\|_2^2}{\sigma^2}\sim\chi^2_{n-j}(\dfrac{\|({\bf I}_n-{\bf P}){\bf u}\|_2^2}{\sigma^2})$. Consider the two full rank sub matrices  ${\bf X}_{\mathcal{J}_1}$ and ${\bf X}_{\mathcal{J}_2}$  formed by columns of ${\bf X}$ indexed by  $\mathcal{J}_1 \subset \mathcal{J}_2$. Let  ${\bf P}_{\mathcal{J}_1}$ and ${\bf P}_{\mathcal{J}_2}$ represent the projection matrices onto the column space of ${\bf X}_{\mathcal{J}_1}$ and ${\bf X}_{\mathcal{J}_2}$ respectively. Then for any R.V ${\bf z}\sim\mathcal{N}({\bf u},\sigma^2{\bf I}_n)$,  $\dfrac{\|({\bf P}_{\mathcal{J}_1}-{\bf P}_{\mathcal{J}_2}){\bf z}\|_2^2}{\sigma^2}\sim\chi^2_{|\mathcal{J}_2|-|\mathcal{J}_1|}(\dfrac{\|({\bf P}_{\mathcal{J}_1}-{\bf P}_{\mathcal{J}_2}){\bf u}\|_2^2}{\sigma^2})$.
 \end{lemma}
 Next we state a frequently used convergence result  \cite{tsp}.
 \begin{lemma}\label{chi2convergence}
 Let ${ z} \sim \chi^2_j(\dfrac{\lambda}{\sigma^2})$, where $\lambda>0$ is a constant w.r.t $\sigma^2$. Then $\sigma^2 z \overset{P}{\rightarrow} \lambda$ as $\sigma^2 \rightarrow 0$.
 \end{lemma}
 \subsection{High SNR consistency: Definition}
 { The high SNR consistency results available in literature\cite{ding2011inconsistency,tsp} deals with full rank linear regression models. Since, uniqueness issues are absent when $\text{rank}({\bf X})=p$, this definition of high SNR consistency demands that $PE \rightarrow 0$ as $\sigma^2\rightarrow 0$ for every signal  $\beta \in \mathbb{R}^p$.  In this article, we relax this definition to account for the uniqueness issues present in  regression models with $n<p$ using the concept of regression class. A regression class $\mathcal{C}$ is defined as the collection of matrix signal pairs (${\bf X},\beta$) where perfect recovery is possible for a particular algorithm under noiseless conditions. For $l_0$-penalty, $\mathcal{C}_1=\{({\bf X} ,\beta) : spark(X)>2|supp(\beta)|\}$ is a regression class. Similarly, $\mathcal{C}_2=\{({\bf X},\beta ): \mu_{\bf X}\leq \dfrac{1}{2|supp(\beta)|-1} \}$  and $\mathcal{C}_3=\{{\bf X},\beta: erc({\bf X},supp(\beta))<1\}$ forms regression classes for $l_1$-penalty, $l_1$-error and OMP. We now formally define  high SNR consistency in underdetermined regression models.
  
{\bf Definition 4:-} A SSP is said to be high SNR consistent for a regression class $\mathcal{C}$ if $PE=\mathbb{P}(\hat{\mathcal{I}}\neq \mathcal{I})$ converges to zero as $\sigma^2 \rightarrow 0$ for every matrix vector pair $({\bf X},\beta) \in \mathcal{C}$.
 
In words, a SSP is high SNR consistent if it can deliver a $PE$ arbitrarily close to zero by decreasing the noise variance $\sigma^2$. Even though every signal in a  regression class can be perfectly recovered under noiseless conditions $(\sigma^2=0)$, to achieve a near perfect recovery at high SNR (i.e., $\sigma^2 \neq 0$, but close to zero), the tuning parameters for the SSPs need to be selected appropriately. In the following sections, we discuss the conditions on the tuning parameters such that the support can be  recovered  with arbitrary precision as $\sigma^2$ decreases. } 
 \section{ High SNR consistency of $l_0$-penalty based SSP.} 
In this section, we describe the high SNR behaviour of $\hat{\beta}=\underset{{\bf b} \in \mathbb{R}^p}{\arg\min} ||{\bf y}-{\bf X}{\bf b}||_2^2+\Gamma_0 \sigma^2||{\bf b}||_0$ and $\hat{\mathcal{I}}=supp(\hat{\beta})$, where the tuning parameter  $\Gamma_0$ is a deterministic positive quantity. The values of $\Gamma_0 $ discussed in the  literature  includes the Akaike information criteria (AIC) with $\Gamma_0 =2$, minimum description length (MDL) or  Bayesian information criteria (BIC) with $\Gamma_0 =\log(n)$, risk inflation criteria (RIC) of Foster and George (RIC-FG) with $\Gamma_0 =2\log(p)$\cite{RIC-FG}, RIC of Zhang and Shen (RIC-ZS) with $\Gamma_0 =2\log(p)+2\log(\log(p))$ \cite{RIC-ZS}, extended Bayesian information criterion (EBIC) with $\Gamma_0=\log(n)+\dfrac{2\gamma}{\|{\bf b}\|_0}\log({p \choose \|{\bf b}\|_0})$\cite{E-BIC} etc. The hyper parameter $\gamma$ in EBIC is a user defined parameter. Under a set of regularity conditions on the matrix ${\bf X}$ and $\beta$,  it was shown that $l_0$-penalty is large sample consistent if, $\Gamma_0=o(n^{c_2-c_1})$, $k^*\log(p)=o(n^{c_2-c_1})$ and $\Gamma_0-2\log(p)-\log(\log(p)) \rightarrow \infty$ as $n \rightarrow \infty$. Here, $c_1$ and $c_2$ are parameters depending on the regularity conditions\cite{GICconsistent}. This result hold true for $(n,p,k^*)\rightarrow \infty$ and $n<p$ or $n\ll p$. Note that these tuning parameters are derived based on the large sample behaviour of $l_0$-penalty. The conditions for high SNR consistency of $l_0$-penalty are not discussed in the literature to the best of our knowledge.  
Next we state and prove the sufficient conditions for the high SNR consistency of $l_0$-penalty.
\begin{thm}\label{Sparktheorem}
Consider a matrix ${\bf X}$ which satisfies $spark({\bf X})>2k^*$. Then for any $k^*$-sparse signal $\beta$, $l_0$-penalty is high SNR consistent if $\underset{\sigma^2 \rightarrow 0}{\lim}\  \Gamma_0 =\infty$ and $\underset{\sigma^2\rightarrow 0}{\lim}\  \sigma^2 \Gamma_0 =0$.
\end{thm}\squeezeup
\begin{proof}
 The optimization problem in $l_0$-penalty can be stated more explicitly as $\hat{\mathcal{I}}=\underset{\mathcal{J} \subset [p]}{\arg\min}\ {L(\mathcal{J})}$, where $L(\mathcal{J})=\underset{{\bf b}:supp({\bf b})=\mathcal{J}}{\min} \|{\bf y}-{\bf X}{\bf b}\|_2^2+\Gamma_0 \sigma^2|\mathcal{J}|$. When ${\bf X}_{\mathcal{J}}$ has full rank, the solution to   $\underset{{\bf b}:supp({\bf b})=\mathcal{J}}{\min} \|{\bf y}-{\bf X}{\bf b}\|_2^2 $ = $\underset{{\bf a} \in \mathbb{R}^{|\mathcal{J}|}}{\min} \|{\bf y}-{\bf X}_{\mathcal{J}}{\bf a}\|_2^2 $  is unique and equal to $\hat{\bf a}=({\bf X}_{\mathcal{J}}^T{\bf X}_{\mathcal{J}})^{-1}{\bf X}_{\mathcal{J}}^T{\bf y}$. In this case,  ${\bf X}_{\mathcal{J}}\hat{\bf a}={\bf P}_{\mathcal{J}}{\bf y}$ and $\underset{{\bf b}:supp({\bf b})=\mathcal{J}}{\min} \|{\bf y}-{\bf X}{\bf b}\|_2^2 $ is equal to $\|({\bf I}_n-{\bf P}_{\mathcal{J}}){\bf y}\|_2^2$. Here ${\bf P}_{\mathcal{J}}={\bf X}_{\mathcal{J}} ({\bf X}_{\mathcal{J}}^T{\bf X}_{\mathcal{J}})^{-1}{\bf X}_{\mathcal{J}}$ is a projection matrix of rank $|\mathcal{J}|=rank({\bf X}_{\mathcal{J}})$.  
When ${\bf X}_{\mathcal{J}}$ is rank deficient, the solution to   $\underset{{\bf b}:supp({\bf b})=\mathcal{J}}{\min} \|{\bf y}-{\bf X}{\bf b}\|_2^2 $ = $\underset{{\bf a} \in \mathbb{R}^{|\mathcal{J}|}}{\min} \|{\bf y}-{\bf X}_{\mathcal{J}}{\bf a}\|_2^2 $ 
can be any one of the infinitely many vectors $\hat{\bf a}$ that solves ${\bf X}_{\mathcal{J}}^T{\bf X}_{\mathcal{J}}\hat{\bf a}={\bf X}_{\mathcal{J}}^T{\bf y}$. A typical solution is denoted by
$\hat{\bf a}=  ({\bf X}_{\mathcal{J}}^T{\bf X}_{\mathcal{J}})^{-}{\bf X}_{\mathcal{J}}^T{\bf y}$, where $({\bf X}_{\mathcal{J}}^T{\bf X}_{\mathcal{J}})^{-}$ is called the generalized inverse of ${\bf X}_{\mathcal{J}}^T{\bf X}_{\mathcal{J}}$\cite{yan2009linear}. The matrix ${\bf X}_{\mathcal{J}} ({\bf X}_{\mathcal{J}}^T{\bf X}_{\mathcal{J}})^{-}{\bf X}_{\mathcal{J}}$ satisfies all the properties of a projection matrix of $rank(\mathcal{X}_{\mathcal{J}})$. We denotes this matrix by ${\bf P}_{\mathcal{J}}$ itself with a caveat that $rank({\bf P}_{\mathcal{J}})=rank({\bf X}_{\mathcal{J}})<|\mathcal{J}|$. With this convention, when ${\bf X}_{\mathcal{J}}$ is rank deficient, $\underset{{\bf b}:supp({\bf b})=\mathcal{J}}{\min} \|{\bf y}-{\bf X}{\bf b}\|_2^2 =\|({\bf I}_n-{\bf P}_{\mathcal{J}}){\bf y}\|_2^2$. Hence, $l_0$-penalty can be reformulated as
 \begin{equation}\label{pls_alt}
\hat{\mathcal{I}}=\underset{\mathcal{J}\subseteq [p]}{\arg\min}{L(\mathcal{J})}=\underset{\mathcal{J}\subseteq [p]}{\arg\min}\|({\bf I}_n-{\bf P}_{\mathcal{J}}){\bf y}\|_2^2+\sigma^2\Gamma_0|\mathcal{J}|.
\end{equation} 
Define the  error event $\mathcal{E}=\{\hat{\mathcal{I}}\neq \mathcal{I}\}=\{\exists \mathcal{J}\in [p]: L(\mathcal{J})\leq L(\mathcal{I}) \}$. Applying union bound to $PE=\P(\mathcal{E})$ gives
 \begin{equation}\label{sum_all}
 \begin{array}{ll}
 PE&\leq\sum\limits_{\mathcal{J}\in {[p]}}\P(L(\mathcal{J})\leq L(\mathcal{I})).\\
 &=\overset{P_1}{\overbrace{\sum\limits_{\mathcal{J} \in {\mathcal{H}}_1}\P(L(\mathcal{J})\leq L(\mathcal{I}))}}+\overset{P_2}{\overbrace{\sum\limits_{\mathcal{J} \in {\mathcal{H}_2}}\P(L(\mathcal{J})\leq L(\mathcal{I}))}}.
 \end{array}
 \end{equation}
 where ${\mathcal{H}_1}=\{\mathcal{J}\in [p]:({\bf I}_n-{\bf P}_{{\mathcal{J}}}){\bf X}{\beta}\neq {\bf 0}_n\} $ and ${\mathcal{H}}_2=\{\mathcal{J}\in [p]:({\bf I}_n-{\bf P}_{{\mathcal{J}}}){\bf X}{\beta}= {\bf 0}_n\}$. In words, ${\mathcal{H}_1}$ represent the subsets $\mathcal{J} \subseteq [p]$ such that the  $col({\bf X}_{\mathcal{J}})$  does not cover the signal subspace $col({\bf X}_{\mathcal{I}})$. For  $\mathcal{I}=\{1,2\}$,  assuming that the columns ${\bf X}_1$, ${\bf X}_2$ and  ${\bf X}_3$ are linearly independent, the subsets $\mathcal{J}=\{1\}$, $\mathcal{J}=\{3\}$, $\mathcal{J}=\{1,3\}$ etc. belongs to $\mathcal{H}_1$. Similarly, ${\mathcal{H}_2}$ represents the subsets $\mathcal{J} \subseteq [p]$ such that the  $col({\bf X}_{\mathcal{J}})$   cover the signal subspace $col({\bf X}_{\mathcal{I}})$. For $\mathcal{I}=\{1,2\}$,   $\mathcal{J}=\{1,2,3\}$, $\mathcal{J}=\{1,2,3, 4\}$ etc. will belong to $\mathcal{H}_2$.  We consider both these summations separately. 

{\bf Case 1 \ $({\bf I}_n-{\bf P}_{\mathcal{J}}){\bf X}{\beta}\neq {\bf 0}_n$:-} 
In this case, it can happen that $|\mathcal{J}|>k^*$, $|\mathcal{J}|=k^*$ or $|\mathcal{J}|<k^*$. Since $\mathcal{I}=supp(\beta)$,  $({\bf I}_n-{\bf P}_{{\mathcal{I}}}){\bf X}{\beta}= {\bf 0}_n$. Thus, by Lemma \ref{chisquarelemma}, $A_1=\dfrac{\|({\bf I}_n-{\bf P}_{{\mathcal{I}}}){\bf y}\|_2^2}{\sigma^2}\sim \chi^2_{n-k^*}$.  Likewise, $({\bf I}_n-{\bf P}_{\mathcal{J}}){\bf X}{\beta}\neq {\bf 0}_n$ implies that $A_2=\dfrac{\|({\bf I}_n-{\bf P}_{\mathcal{J}}){\bf y}\|_2^2}{\sigma^2}\sim \chi^2_{n-rank({\bf X}_{\mathcal{J}})}(\dfrac{\lambda_{\mathcal{J}}}{\sigma^2})$, where $\lambda_{\mathcal{J}}={\|({\bf I}_n-{\bf P}_{\mathcal{J}}){\bf X}{\beta}\|_2^2}>0$. Hence,
\begin{equation}
\begin{array}{ll}
\P(\mathcal{E}_{\mathcal{J}})&=\P\left(L(\mathcal{J})<L({\mathcal{I}})\right)\\&=\P\left((A_2-A_1)\sigma^2+\Gamma_0 \sigma^2(|\mathcal{J}|-k^*)<0\right).
\end{array}
\end{equation}
Since, $A_1\sim \chi^2_{n-k^*}$ is a B.I.P R.V, $A_1\sigma^2 \overset{P}{\rightarrow} 0$ as $\sigma^2 \rightarrow 0$.  By Lemma \ref{chi2convergence}, $\sigma^2A_2 \overset{P}{\rightarrow}\lambda_{\mathcal{J}}>0$ as $\sigma^2 \rightarrow 0$.  By the hypothesis of Theorem \ref{Sparktheorem}, $\Gamma_0\sigma^2(|\mathcal{J}|-k^*) \rightarrow 0$ as $\sigma^2 \rightarrow 0$. This implies that $(A_2-A_1)\sigma^2+\Gamma_0\sigma^2(|\mathcal{J}|-k^*)\overset{P}{\rightarrow}\lambda_{\mathcal{J}}>0$. Now, by the definition of C.I.P, for any $\epsilon>0$, $\exists  \sigma^2_{\mathcal{J}}>0$ such that $\P\left(|(A_2-A_1)\sigma^2+\Gamma_0\sigma^2(|\mathcal{J}|-k^*)-\lambda_{\mathcal{J}}|>\dfrac{\lambda_{\mathcal{J}}}{2}\right)<\epsilon$, for all $\sigma^2<\sigma^2_{\mathcal{J}}$. This implies that
 \begin{equation}
\P(\mathcal{E}_{\mathcal{J}})\leq \P\left((A_2-A_1)\sigma^2+\Gamma_0\sigma^2(|\mathcal{J}|-k^*)<\dfrac{\lambda_{\mathcal{J}}}{2}\right)\leq \epsilon, 
\end{equation}
$\forall \sigma^2<\sigma^2_{\mathcal{J}}$. Thus, $\underset{\sigma^2 \rightarrow 0}{\lim}\P(\mathcal{E}_{\mathcal{J}})=0$, $\forall \mathcal{J}\in {\mathcal{H}}_1$. This together with $|{\mathcal{H}}_1|<\infty$ implies that $\underset{\sigma^2 \rightarrow 0}{\lim}P_1=0$. \\
{\bf Case 2 \ $({\bf I}_n-{\bf P}_{{\mathcal{J}}}){\bf X}{\beta}={\bf 0}_n$:-} $spark({\bf X})>2k^*$  implies  that $\beta$ is the sparsest solution to the equation ${\bf X}{\bf b}={\bf X}\beta$. Hence, $({\bf I}_n-{\bf P}_{\mathcal{J}}){\bf X}{\beta}= {\bf 0}_n$ implies that $|\mathcal{J}|>k^*$. Since, $({\bf I}_n-{\bf P}_{\mathcal{J}}){\bf X}{\beta}= {\bf 0}_n$,  $A_2=\dfrac{\|({\bf I}_n-{\bf P}_{\mathcal{J}}){\bf y}\|_2^2}{\sigma^2}\sim \chi^2_{n-rank({\bf X}_{\mathcal{J}})}$. Thus $\P(\mathcal{E}_{\mathcal{J}})$ becomes
\begin{equation}
\P(\mathcal{E}_{\mathcal{J}})=\P\left(L(\mathcal{J})<L({\mathcal{I}})\right)=\P\left((A_1-A_2)>\Gamma_0(|\mathcal{J}|-k^*)\right).
\end{equation}
Note that both $A_1$ and $A_2$ are B.I.P R.Vs with distribution independent of $\sigma^2$ and so is $A_1-A_2$. Thus, $\exists t_{\epsilon}<\infty$ independent of $\sigma^2$ such that $\P(A_1-A_2>t_{\epsilon})<\epsilon$. Since, $|\mathcal{J}|>k^*$, by the hypothesis of Theorem \ref{Sparktheorem}, $\Gamma_0(|\mathcal{J}|-k^*)\rightarrow \infty$ as $\sigma^2\rightarrow 0$. Thus, $\exists \sigma^2_{\mathcal{J}}>0$, such that $\Gamma_0(|\mathcal{J}|-k^*)>t_{\epsilon}$, $\forall \sigma^2<\sigma^2_{\mathcal{J}}$. Combining, we get  $\P(\mathcal{E}_{\mathcal{J}})<\epsilon,\ \forall \sigma^2<\sigma^2_{\mathcal{J}}$. Thus, $\underset{\sigma^2 \rightarrow 0}{\lim}\P(\mathcal{E}_{\mathcal{J}})=0$, $\forall \mathcal{J}\in {\mathcal{H}_2}$. This together with $|{\mathcal{H}}_2|<\infty$ implies that $\underset{\sigma^2 \rightarrow 0}{\lim}P_2=0$. Thus, under the hypothesis of Theorem \ref{Sparktheorem}, $l_0$-penalty is high SNR consistent. 
\end{proof}
\begin{remark}
{ Theorem \ref{Sparktheorem}  details a range of SNR adaptations on $\Gamma_0 $ such that $l_0$-penalty is high SNR consistent. However, different  SNR adaptations satisfying Theorem \ref{Sparktheorem} leads to different convergence rates of $PE$. The proof of Theorem \ref{Sparktheorem} reveals that $P_1$ is related to the probability of underestimation and $P_2$ is related to the probability of overestimation in MOS problems detailed in \cite{tsp}. To summarise,  $\Gamma_0$ with faster rate of increase to $\infty$ will have lower values of $P_2$ and higher values of $P_1$ and vice versa. }
\end{remark}
\subsection{ High SNR consistency of $l_0$-penalty: Necessary conditions}
The SNR adaptations required by Theorem \ref{Sparktheorem} are in sharp contrast to the $\sigma^2$ independent values of $\Gamma_0 $ discussed in literature. The following theorem  proves that $l_0$-penalty with $\sigma^2$ independent values of $\Gamma_0$ are inconsistent at high SNR.
\begin{thm}\label{sparknecessity1}
Consider a matrix ${\bf X}$ with $spark({\bf X})>2k^*$. Then for any $k^*$-sparse vector $\beta$, $l_0$-penalty is high SNR consistent only if $\underset{\sigma^2 \rightarrow 0}{\lim} \Gamma_0=\infty$.  
\end{thm}
\begin{proof}
Define $\mathcal{J}=\mathcal{I} \cup i$, where $i\notin \mathcal{I}$. Note that $ |\mathcal{J}|= k^*+1\leq 2k^*$, $\forall k^*\geq 1$ and $|\mathcal{J}|=1$ if $k^*=0$. Further for any matrix ${\bf X}$, $spark({\bf X})\geq 2$. Hence, $spark({\bf X})>2k^*$ implies that ${\bf X}_{\mathcal{J}}$ has full rank for $k^*\geq 0$. This together with $\mathcal{I} \subset \mathcal{J}$ implies that $({\bf I}_n-{\bf P}_{{\mathcal{J}}}){\bf X}{\beta}= {\bf 0}_n$. Expanding  $L(\mathcal{J})=\|({\bf I}_n-{\bf P}_{\mathcal{J}}){\bf y}\|_2^2+\sigma^2\Gamma_0|\mathcal{J}|$  and applying Lemma \ref{chisquarelemma}, we have
\begin{equation}
PE \geq \P\left(L(\mathcal{J})<L(\mathcal{I})\right)\geq \P(A>\Gamma_0),
\end{equation}
where $A=\dfrac{{\bf y}^T({\bf P}_{\mathcal{J}}-{\bf P}_{\mathcal{I}}){\bf y}}{\sigma^2}\sim \chi^2_1$, $\forall \sigma^2>0$.  $A \sim \chi^2_1$ implies that $A=Z^2$, where $Z\sim \mathcal{N}(0,1)$. Thus, $PE \geq \P(A>\Gamma_0)=\P(|Z|>\sqrt{\Gamma_0})=2Q(\sqrt{\Gamma_0}), \ \forall \sigma^2>0$. Here $Q(x)=\dfrac{1}{\sqrt{2\pi}}\int_{t=x}^{\infty}\exp({-\frac{t^2}{2}})dt$ is the complementary cumulative distribution function of a $N(0,1)$ R.V. Hence, $l_0$-penalty is high SNR consistent only if $\underset{\sigma^2 \rightarrow 0}{\lim}\Gamma_0=\infty$. 
\end{proof}
\begin{remark}
It follows directly from the proof of Theorem \ref{sparknecessity1} that PE of $l_0$-penalty with SNR independent $\Gamma_0$ like BIC, AIC etc. satisfy ${PE} \geq 2Q(\sqrt{\Gamma_0})$, $\forall \sigma^2>0$.  For RIC-FG with $\Gamma_0=2\log(p)$, the lower bound  $2Q(\sqrt{\Gamma_0})$ will be less than $0.01$ only for $p\geq 28$ and $2Q(\sqrt{\Gamma_0})\leq 0.001$ only for $p\geq 225$.  Hence, the performance of these criteria in small and medium sized problems will be  suboptimal. 
\end{remark}
Theorems 2  implies that $\underset{\sigma^2 \rightarrow 0}{\lim}\  \Gamma_0 =\infty$ is a necessary  condition  for high SNR consistency. We next establish the necessity of $\underset{\sigma^2\rightarrow 0}{\lim}\  \sigma^2 \Gamma_0 =0$  for high SNR consistency. 
\begin{thm}\label{sparknecessity2}
Consider a matrix ${\bf X}$ which satisfies $spark({\bf X})>2k^*$. Then for any $k^*$-sparse signal $\beta$ with $k^*\geq 1$, $l_0$-penalty  is high SNR consistent only if $\underset{\sigma^2 \rightarrow 0}{\lim} \sigma^2\Gamma_0=0$.  
\end{thm} 

\begin{proof}
Define $\mathcal{J}=\mathcal{I}/ i$, where $i\in \mathcal{I}$. Since, $\mathcal{J}\subset\mathcal{I}$ and $spark({\bf X})>2k^*$, it follows that $({\bf I}_n-{\bf P}_{{\mathcal{J}}}){\bf X}{\beta}\neq{\bf 0}_n$.  Expanding $L(\mathcal{J})=\|({\bf I}_n-{\bf P}_{\mathcal{J}}){\bf y}\|_2^2+\sigma^2\Gamma_0|\mathcal{J}|$ and applying Lemma \ref{chisquarelemma}, we have
\begin{equation}
PE \geq \P\left(L(\mathcal{J})<L(\mathcal{I})\right)\geq \P(A<\Gamma_0\sigma^2),
\end{equation}
where $A={\bf y}^T({\bf P}_{\mathcal{I}}-{\bf P}_{\mathcal{J}}){\bf y}\sim \sigma^2\chi^2_1(\dfrac{\lambda}{\sigma^2})$ with $\lambda=\|({\bf P}_{\mathcal{I}}-{\bf P}_{\mathcal{J}}){\bf X}\beta\|_2^2>0$. By Lemma \ref{chi2convergence}, $A\overset{P}{\rightarrow}\lambda$ as $\sigma^2 \rightarrow 0$. Suppose that $\underset{\sigma^2 \rightarrow 0}{\lim}\sigma^2\Gamma_0=\lambda_1$, where $\lambda_1>\lambda$. Then, $\exists \sigma^2_1>0$ such that $\dfrac{\lambda_1+\lambda}{2}<\sigma^2\Gamma_0<\lambda_1$, $\forall \sigma^2<\sigma^2_1$. This implies that 
\begin{equation}
\begin{array}{ll}
\P(A<\Gamma_0\sigma^2)&\geq \P(A<\dfrac{\lambda_1+\lambda}{2})=1-\P(A-\lambda>\dfrac{\lambda_1-\lambda}{2}) \\
& \geq  1-\P(|A-\lambda|>\dfrac{\lambda_1-\lambda}{2}), \ \forall \sigma^2<\sigma^2_1.
\end{array}
\end{equation} 
Since, $A \overset{P}{\rightarrow} \lambda$ as $\sigma^2\rightarrow 0$, for any $\epsilon>0$, $\exists \sigma^2_2>0$ such that $\P(|A-\lambda|>\dfrac{\lambda_1-\lambda}{2})\leq \epsilon,\ \forall \sigma^2<\sigma^2_2$. Fix $\sigma^2(\epsilon)=\min(\sigma^2_1,\sigma^2_2)$. Then $\forall \sigma^2<\sigma^2(\epsilon), PE \geq 1-\epsilon$. Thus if $\lambda_1>\lambda$, then $\underset{\sigma^2 \rightarrow 0}{\lim}PE=1$. This implies that $\underset{\sigma^2 \rightarrow 0}{\lim}\sigma^2\Gamma_0<\lambda$ is a necessary condition for high SNR consistency. However, without \textit{a priori} knowledge of  non-zero entries of $\beta$, $\lambda$ is unknown.
Hence, $l_0$-penalty is high SNR consistent  only if $\underset{\sigma^2 \rightarrow 0}{\lim} \sigma^2\Gamma_0=0$.
\end{proof}
\begin{remark}
{The formulation of $l_0$-penalty given in (\ref{pls_alt}) is exactly similar to that of MOS problems given in (\ref{pll_mos}) except that the search space of MOS is a very small subset of the search space in $l_0$-penalty. This is reflected in the similarity of NSCs for MOS derived in \cite{tsp} and Theorems \ref{Sparktheorem}-\ref{sparknecessity2} for subset selection. It is also true that different values of $\Gamma_0$ gives EEF, NMDL etc. as special cases. Hence, Theorems  \ref{Sparktheorem}-\ref{sparknecessity2} can be seen as an extension of the existing high SNR consistency  results in \cite{tsp,ding2011inconsistency,SNLShighSNR,schmidt2012consistency} to subset selection problems. However, the novelty of Theorems \ref{Sparktheorem}-\ref{sparknecessity2} lies in the fact that it explicitly takes into account the identifiability issues associated with subset selection in underdetermined linear models.  These structural issues were not considered in \cite{tsp,ding2011inconsistency,SNLShighSNR,schmidt2012consistency} which dealt with MOS in overdetermined linear regression models.    }
\end{remark}

\section{High SNR consistency of convex relaxation based SSPs}
In this section, we derive NSCs on the tuning parameters $\{\Gamma_i\}_{i=1}^3
$ such that $l_1$-penalty, $l_1$-error and DS are high SNR consistent. Unlike the NP-hard $l_0$-penalty which is computationally infeasible except in small sized problems, the CR based SSPs discussed in this section and the greedy algorithms like OMP discussed in Section \rom{5}  can be implemented with polynomial complexity. Hence, these techniques are practically important. {Unlike the high SNR consistency of $l_0$-penalty whose connections with the high SNR consistency in MOS problems we previously mentioned, the high SNR consistency of CR and greedy algorithms are not discussed in open literature to the best of our knowledge.}  We first discuss the $l_1$-penalty based SSP.
\subsection{ High SNR consistency of $l_1$-penalty: Sufficient conditions}
In this section, we discuss the high SNR behaviour of   
$ \hat{\beta}=\underset{{\bf b} \in \mathbb{R}^p}{\arg\min} \dfrac{1}{2}\|{\bf y}-{\bf X}{\bf b}\|_2^2+\Gamma_1\sigma \|{\bf b}\|_1 $
 and $\hat{\mathcal{I}}=supp(\hat{\beta})$.  This is a widely used  SSP in high dimensional statistics. $l_1$-penalty is the convex program that is closest to the optimal but NP-hard $l_0$-penalty. Commonly used values of $\Gamma_1$  include $\Gamma_1=2\sqrt{2\log(p)}$\cite{candes2009near}, $\Gamma_1=\sqrt{8(1+\eta)\log(p-k^*)}$ \cite{ben2010coherence},  $\Gamma_1=10\sqrt{\log(p)}$ \cite{candes2011probabilistic}  etc. Here, $\eta>0$ is a constant. The large sample consistency of $l_1$-penalty is also widely studied. For a fixed  $p$ and $k^*$,  all values of $\Gamma_1$ satisfying $\dfrac{\Gamma_1}{n}\rightarrow 0$ and $\dfrac{\Gamma_1}{n^{\frac{1+c}{2}}} \rightarrow \infty$ as $n \rightarrow \infty$ results in large sample consistency under a set of regularity conditions \cite{zhao2006model}. $c$ depends on these regularity conditions. However, the consistency of $l_1$-penalty as $\sigma^2 \rightarrow 0$ is not discussed in literature to the best of our knowledge. Next we state and prove the sufficient conditions  for the high SNR consistency of $l_1$-penalty.
\begin{thm}\label{l1_penalty_sufficient}
 $l_1$-penalty is high SNR consistent for any matrix signal pair $({\bf X},\beta)$ satisfying the ERC provided that the tuning parameter $\Gamma_1$ satisfies $\underset{\sigma^2 \rightarrow 0}{\lim}\Gamma_1=\infty$ and $\underset{\sigma^2 \rightarrow 0}{\lim}\sigma{\Gamma_1}=0$.
\end{thm}  
\begin{proof}
The proof of Theorem \ref{l1_penalty_sufficient} is based on the following  fundamental result proved in [Theorem 8,\cite{tropp2006just}].
\begin{lemma} \label{l1_penalty_relax}
Let $\mathcal{J}$ be any index set satisfying ERC. If ${\bf y}^{\mathcal{J}}={\bf P}_{\mathcal{J}}{\bf y}$ satisfies $\|{\bf X}^T({\bf y}-{\bf y}^{\mathcal{J}})\|_{\infty}< \sigma \Gamma_1\left(1-erc({\bf X},\mathcal{J})\right)$, then $\hat{\beta}$ satisfies the following. \\
A1). $supp(\hat{\beta}) \subseteq \mathcal{J}$. \\
A2). $\hat{\beta}$ is the unique minimizer of $l_1$-penalty. \\
A3). $\mathcal{T}=\{j : |{\bf b}^{\mathcal{J}}(j)|>\Gamma_1 \sigma\|({\bf X}_{\mathcal{J}}^T{\bf X}_{\mathcal{J}})^{-1}\|_{\infty,\infty}\}\subseteq supp(\hat{\beta})$, where ${\bf b}^{\mathcal{J}}={\bf X}_{\mathcal{J}}^{\dagger}{\bf y}$ is the LS estimate of $\beta_{\mathcal{J}}$.
\end{lemma} 
{ In words, Lemma \ref{l1_penalty_relax} states that if the correlation between the columns in ${\bf X}$ and residual generated by the LS fit using the columns in $\mathcal{J}$ is sufficiently low, then  the support of solution to $l_1$-penalty will be contained in $\mathcal{J}$. Further, $l_1$-penalty does not miss indices that has sufficiently large values in the restricted LS estimate ${\bf b}^{\mathcal{J}}$}. By the hypothesis of Theorem \ref{l1_penalty_sufficient}, the true support $\mathcal{I}$ satisfies $erc({\bf X},\mathcal{I})<1$.   Thus, if the event ${\mathcal{E}}_1=\{\|{\bf X}^T({\bf y}-{\bf y}^{\mathcal{I}})\|_{\infty}<\Gamma_1\sigma \left(1-erc({\bf X},\mathcal{I})\right)\}$ is true, then $supp(\hat{\beta})\subseteq \mathcal{I}$. That is, $l_1$-penalty does not make any false discoveries.  If the event ${\mathcal{E}_2}=\{\forall j : |{\bf b}^{\mathcal{I}}(j)|>\Gamma_1\sigma \|({\bf X}_{\mathcal{I}}^T{\bf X}_{\mathcal{I}})^{-1}\|_{\infty,\infty}\}=\{|\mathcal{T}|=k^*\}$ is also true, then $supp(\hat{\beta})=\mathcal{I}$. Thus $\P(\hat{\mathcal{I}}=\mathcal{I})\geq \P(\mathcal{E}_1\cap \mathcal{E}_2)$. 

We first analyse the probability of the event $\mathcal{E}_1$.
 Note that $({\bf I}_n-{\bf P}_{\mathcal{I}}){\bf X\beta}=({\bf I}_n-{\bf P}_{\mathcal{I}}){\bf X}_{\mathcal{I}}{\beta}_{\mathcal{I}}=0$. Hence $\|{\bf X}^T({\bf I}_n-{\bf P}_{\mathcal{I}}){\bf y}\|_{\infty}=\|{\bf X}^T({\bf I}_n-{\bf P}_{\mathcal{I}}){\bf w}\|_{\infty}$. Further, $\|{\bf X}_j\|_2=1$ and Cauchy Schwartz inequality implies that $\underset{j}{\max}|{\bf X}^T_j({\bf I}_n-{\bf P}_{\mathcal{I}}){\bf w}|\leq \underset{j}{\max}\|{\bf X}_j||_2||({\bf I}_n-{\bf P}_{\mathcal{I}}){\bf w}\|_2=\|({\bf I}_n-{\bf P}_{\mathcal{I}}){\bf w}\|_2 $. Using these inequalities, we can bound $\P(\mathcal{E}_1)$ as
\begin{equation}\label{l1_penalty 1}
\begin{array}{ll}
\P(\mathcal{E}_1)&=\P\left(\underset{j}{\max}|{\bf X}_j^T({\bf I}_n-{\bf P}_{\mathcal{I}}){\bf w}|<\Gamma_1\sigma\left(1-erc(X,\mathcal{I})\right)\right) \\
&\geq \P\large(\|({\bf I}_n-{\bf P}_{\mathcal{I}}){\bf w}\|_2<\Gamma_1\sigma\left(1-erc(X,\mathcal{I})\right)\large) \\
&=\P\left(\dfrac{\|({\bf I}_n-{\bf P}_{\mathcal{I}}){\bf w}\|_2^2}{\sigma^2}<\Gamma_1^2\left(1-erc(X,\mathcal{I})\right)^2\right) 
\end{array}
\end{equation}
Note that $\dfrac{\|({\bf I}_n-{\bf P}_{\mathcal{I}}){\bf w}\|_2^2}{\sigma^2}\sim \chi^2_{n-k^*}$ is a B.I.P R.V with distribution independent of $\sigma^2$.
Hence, if the condition $\underset{\sigma^2 \rightarrow 0}{\lim}{\Gamma_1}=\infty$ in the hypotheses of Theorem \ref{l1_penalty_sufficient} is satisfied, then the lower bound in (\ref{l1_penalty 1}) converges to 1. Hence, $\underset{\sigma^2\rightarrow 0}{\lim}\P(\mathcal{E}_1)=1$.

Next, we analyse $\P(\mathcal{E}_2)$. Since $\mathcal{I}$ is the correct support, it follows that ${\bf b}^{\mathcal{I}}={\bf X}_{\mathcal{I}}^{\dagger}({\bf X}_{\mathcal{I}}\beta_{\mathcal{I}}+{\bf w})=\beta_{\mathcal{I}}+{\bf X}_{\mathcal{I}}^{\dagger}{\bf w}$. Since ${\bf w}\sim \mathcal{N}({\bf 0}_n,\sigma^2{\bf I}_n)$, we have ${\bf b}^{\mathcal{I}}\sim \mathcal{N}(\beta_{\mathcal{I}},\sigma^2({\bf X}_{\mathcal{I}}^T{\bf X}_{\mathcal{I}})^{-1})$. The set  $\mathcal{T}$ in A3) of Lemma \ref{l1_penalty_relax}  can be rewritten as $\mathcal{T}=\{j:|{\bf b}^{\mathcal{I}}(j)|>\sigma c_j \Gamma_1 d_j\}$, where $c_j=\sqrt{\left(({\bf X}_{\mathcal{I}}^T{\bf X}_{\mathcal{I}})^{-1}\right)_{j,j}}$ and $d_j=\dfrac{\|({\bf X}_{\mathcal{I}}^T{\bf X}_{\mathcal{I}})^{-1}\|_{\infty,\infty}}{c_j}$.
The NSC for the high SNR consistency of a threshold based SSP like this is given below.
\begin{lemma}\label{l1penalty_necessary_lemma}
Let ${\bf z}\sim  \mathcal{N}({\bf u},\sigma^2{\bf C})$ and $\mathcal{K}=supp({\bf u})$. Consider the threshold based estimator $\hat{\mathcal{K}}=\{j:|{\bf z}_j|>\sigma \sqrt{{\bf C}_{j,j}}\Gamma \}$  of $\mathcal{K}$. Define the event false discovery $\mathcal{F}=\{\exists j \in \hat{\mathcal{K}} \ and \ j \notin \mathcal{K} \}$ and missed discovery $\mathcal{M}= \{\exists j \notin \hat{\mathcal{K}} \ and \ j \in \mathcal{K} \}$. Then the following statements are true\cite{spl}.  \\
L1). $\underset{\sigma^2 \rightarrow 0}{\lim}\mathbb{P}(\mathcal{F})=0, \text{iff} \ \underset{\sigma^2 \rightarrow 0}{\lim}\Gamma=\infty$. \\
L2). $\underset{\sigma^2 \rightarrow 0}{\lim}\mathbb{P}(\mathcal{M})=0, \text{iff} \ \underset{\sigma^2 \rightarrow 0}{\lim}\sigma\Gamma<\underset{j \in \mathcal{K}}{\min}\dfrac{|{\bf u}_j|}{\sqrt{{\bf C}_{j,j}}}$.
\end{lemma} 
Hence, if $\Gamma_1$ satisfies $\underset{\sigma^2 \rightarrow 0}{\lim}\sigma\Gamma_1=0$, then by L2) of Lemma \ref{l1penalty_necessary_lemma}, all entries in $\mathcal{I}$ will be included in $\mathcal{T}$ at high SNR. Mathematically,  $\underset{\sigma^2 \rightarrow 0}{\lim}\P(\mathcal{E}_2)=\underset{\sigma^2 \rightarrow 0}{\lim}\P(|\mathcal{T}|=k^*)=1$. Since, $\underset{\sigma^2 \rightarrow 0}{\lim}\P(\mathcal{E}_1)=1$ and $\underset{\sigma^2 \rightarrow 0}{\lim}\P(\mathcal{E}_2)=1$, it follows that $\underset{\sigma^2 \rightarrow 0}{\lim}\P(\hat{\mathcal{I}}=\mathcal{I})\geq\underset{\sigma^2 \rightarrow 0}{\lim}\P(\mathcal{E}_1\cap \mathcal{E}_2)=1 $.
\end{proof}
\subsection{On the choice of SNR adaptation in $\Gamma_1$.}
Theorem 4 states that all SNR adaptations on $\Gamma_1$ satisfying $\underset{\sigma^2 \rightarrow 0}{\lim}\Gamma_1=\infty$ and $\underset{\sigma^2 \rightarrow 0}{\lim}\sigma\Gamma_1=0$ results in the high SNR consistency of $l_1$-penalty. However, the choice of SNR adaptation has profound influence on the performance of $l_1$-penalty in the moderate to high SNR range. In this section, we derive convergence rates for  $\P(\mathcal{E}_1)$ and $\P(\mathcal{E}_2)$ discussed in the proof of Theorem \ref{l1_penalty_sufficient}. First consider the event ${\mathcal{E}}_1=\{\|{\bf X}^T({\bf y}-{\bf y}^{\mathcal{I}})\|_{\infty}<\Gamma_1\sigma \left(1-erc({\bf X},\mathcal{I})\right)\}$. Following (\ref{l1_penalty 1}), we have 
\begin{equation}\label{conv1}
\P(\mathcal{E}_1)\geq 1-\P\left(A>\Gamma_1^2\left(1-erc(X,\mathcal{I})\right)^2\right), 
\end{equation}
where $A \sim \chi^2_{n-k^*}$. Let $X\sim \chi^2_k$ and $a^2>k$. Then by Lemma 10 in \cite{tsp}, we have 
\begin{equation}\label{conv2}
\P(X>a^2)\leq \dfrac{\exp(\frac{k}{2})}{k^{\frac{k}{2}}}\exp\left(\dfrac{-1}{2}[a^2-k\log(a^2)]\right).
\end{equation}
Let $b_1=1-erc({\bf X},\mathcal{I})$ and $b_2=\frac{\exp(\frac{n-k^*}{2})}{({n-k^*)}^{\frac{n-k^*}{2}}}$.  Applying (\ref{conv2}) in (\ref{conv1}) gives, 
\begin{equation}\label{conv3}
\P(\mathcal{E}_1)\geq 1-b_2\exp\left(\dfrac{-1}{2}[\Gamma_1^2b_1^2-(n-k^*)\log(\Gamma_1^2b_1^2)]\right).
\end{equation}
The R.H.S of inequality in (\ref{conv3}) is independent of $\sigma^2$ for the SNR independent $\Gamma_1$ discussed in literature. Further, the inequality (\ref{conv3}) converges  to one faster as the growth of $\Gamma_1$ increases. Let $\Gamma_1=\dfrac{1}{\sigma^{\alpha}}$  be the SNR adaptation in $\Gamma_1$. This adaptation satisfies Theorem \ref{l1_penalty_sufficient} if $0<\alpha<1$. The convergence rate of $\P(\mathcal{E}_1)$  will be faster for $\alpha_1$ than that of $\alpha_2$ if $\alpha_1>\alpha_2$. 

Next  consider the event ${\mathcal{E}_2}=\{\forall j : |{\bf b}^{\mathcal{I}}(j)|>\Gamma_1\sigma c_jd_j\}$, where ${\bf b}^{\mathcal{I}}={\bf X}_{\mathcal{I}}^{\dagger}{\bf y}$, $c_j= \sqrt{\left(({\bf X}_{\mathcal{I}}^T{\bf X}_{\mathcal{I}})^{-1}\right)_{j,j}}$ and $d_j=\dfrac{\|({\bf X}_{\mathcal{I}}^T{\bf X}_{\mathcal{I}})^{-1}\|_{\infty,\infty}}{c_j}$ as used in the proof of Theorem \ref{l1_penalty_sufficient}.  The following set of inequalities follows directly from union bound and the ${\bf b}^{\mathcal{I}}(j)\sim\mathcal{N}(\beta_j,\sigma^2c_j^2)$ distribution of ${\bf b}^{\mathcal{I}}(j)$.
\begin{equation}\label{conv4}
\begin{array}{ll}
\P(\mathcal{E}_2)&=\P(\underset{j\in [k^*]}{\cap}|{\bf b }^{\mathcal{I}}(j)|>\sigma c_j\Gamma_1 d_j)\\
&\geq 1-\sum\limits_{j=1}^{k^*}\P(|{\bf b }^{\mathcal{I}}(j)|<\sigma c_j\Gamma_1 d_j)\\
&=1-\sum\limits_{j=1}^{k^*} \left[Q(-\Gamma_1d_j-\dfrac{\beta_j}{\sigma c_j})-Q(\Gamma_1d_j-\dfrac{\beta_j}{\sigma c_j})\right].
\end{array}
\end{equation}
Applying  $Q(x)>0,\forall x$, gives
\begin{equation}\label{conv5}
\P(\mathcal{E}_2)\geq 1-\sum\limits_{j=1}^{k^*} Q(-\Gamma_1d_j-\dfrac{\beta_j}{\sigma c_j}).
\end{equation}
For the ease of exposition assume that $\beta_j<0,\forall j\in \mathcal{I}$.
Since, $\underset{\sigma^2 \rightarrow 0}{\lim}\sigma\Gamma_1=0$, we have $\underset{\sigma^2 \rightarrow 0}{\lim}(-\Gamma_1d_j-\dfrac{\beta_j}{\sigma c_j})=\infty$. Hence, $\exists \sigma^2_1>0$ such that $-\Gamma_1d_j-\dfrac{\beta_j}{\sigma c_j}>2,\forall j$. Using the bound $Q(x)\leq \dfrac{1}{2}{\exp\left(-\dfrac{x^2}{2}\right)},\forall x>2$, we have
\begin{equation}\label{conv6}
\P(\mathcal{E}_2)\geq 1-\dfrac{1}{2}\sum\limits_{j=1}^{k^*} {\exp\left(\dfrac{-\left(-\Gamma_1d_j-\dfrac{\beta_j}{\sigma c_j}\right)^2}{2}\right)},
\end{equation}
 $\forall \sigma^2<\sigma^2_1$. Unlike the bound (\ref{conv3}) on $\P(\mathcal{E}_1)$, the R.H.S in (\ref{conv6}) increases with the signal strength $|\beta_j|$. Further, the convergence rate  of $\P(\mathcal{E}_2)$ decreases with the increase in the rate at which  $\Gamma_1$ increase to $\infty$. For $\Gamma_1=\dfrac{1}{\sigma^{\alpha}}$,  the convergence rate of $\P(\mathcal{E}_2)$ decreases with increase in $\alpha$. 
 
 We now make the following observations on the choice of SNR adaptations based on (\ref{conv3}) and (\ref{conv6}). Consider SNR adaptations of the form $\Gamma_1=\dfrac{1}{\sigma^{\alpha}}$. When signal strength  is low, i.e., $\beta_j$ is low for some $j \in \mathcal{I}$, it is reasonable to choose slow rates for $\Gamma_1$ like $\alpha=0.1$. This will ensure the increase of $\P(\mathcal{E}_1)$ to one at a descent rate without causing  significant decrease in the convergence rates of $\P(\mathcal{E}_2)$. However, when the signal strength is high, i.e., $\beta_j$ is high for all $j \in \mathcal{I}$, $\P(\mathcal{E}_2)$ will  be  close to one for moderate values of SNR for most values of $0<\alpha<1$. Then the gain in the convergence rate of $\P(\mathcal{E}_1)$ by allowing a larger value of $\alpha$ will overpower the slight decrease in the convergence rate in $\P(\mathcal{E}_2)$. Hence, when signal strength is high, one can choose faster SNR adaptations like $\alpha=0.5$. 
\subsection{ High SNR consistency of $l_1$-penalty: Necessary conditions}
In the following, we establish the necessity of SNR adaptations detailed in Theorem \ref{l1_penalty_sufficient} for high SNR consistency. 
\begin{thm}\label{l1penalty_necessary}
Suppose  $\exists \mathcal{J}\supset  \mathcal{I}$ such that the matrix support pair $({\bf X},\mathcal{J})$   satisfy ERC. Then $l_1$-penalty is high SNR consistent only if $\underset{\sigma^2 \rightarrow 0}{\lim}\Gamma_1=\infty$.
\end{thm}
\begin{proof}
Let $\mathcal{J}\supset \mathcal{I}$ be an index set satisfying ERC. Define the events ${\mathcal{E}}_1:\{\|{\bf X}^T({\bf y}-{\bf y}^{\mathcal{J}})\|_{\infty}<\Gamma_1\sigma \left(1-erc({\bf X},\mathcal{J})\right)\}$ and $\mathcal{E}_2:\{|\mathcal{T}|=|\mathcal{J}|\}$, where $\mathcal{T}=\{j : |{\bf b}^{\mathcal{J}}(j)|>c_j\sigma\Gamma_1 d_j\}$, $c_j= \sqrt{\left(({\bf X}_{\mathcal{J}}^T{\bf X}_{\mathcal{J}})^{-1}\right)_{j,j}}$ and $d_j=\dfrac{\|({\bf X}_{\mathcal{J}}^T{\bf X}_{\mathcal{J}})^{-1}\|_{\infty,\infty}}{c_j}$. ${\bf y}^{\mathcal{J}}={\bf P}_{\mathcal{J}}{\bf y}$ and ${\bf b}^{\mathcal{J}}={\bf X}_{\mathcal{J}}^{\dagger}{\bf y}$ are the same as in Lemma \ref{l1_penalty_relax}. If both these events are true, then by Lemma \ref{l1_penalty_relax}, $\hat{\mathcal{I}}=\supp(\hat{\beta})=\mathcal{J}\supset \mathcal{I}$. Hence, $PE=\P(\hat{\mathcal{I}}\neq \mathcal{I})\geq \P(\mathcal{E}_1\cap \mathcal{E}_2)$. Since $\mathcal{I}\subset \mathcal{J}$,  ${\bf y}-{\bf y}^{\mathcal{J}}=({\bf I}_n-{\bf P}_{\mathcal{J}}){\bf y}=({\bf I}_n-{\bf P}_{\mathcal{J}}){\bf w}$. Replacing $\mathcal{I}$ with $\mathcal{J}$ in (\ref{l1_penalty 1}), we have 
\begin{equation}
\begin{array}{ll}
\P(\mathcal{E}_1)&=\P\left(\underset{j}{\max}|{\bf X}_j^T({\bf I}_n-{\bf P}_{\mathcal{J}}){\bf w}|<\sigma\Gamma_1(1-erc({\bf X},\mathcal{J}))\right)\\ &\geq \P\left(A<\Gamma_1^2(1-erc({\bf X},\mathcal{J}))^2\right),
\end{array}
\end{equation}
where $A=\dfrac{\|({\bf I}_n-{\bf P}_{\mathcal{J}}){\bf w}\|_2^2}{\sigma^2}\sim \chi^2_{n-|\mathcal{J}|}$ is a R.V with distribution independent of $\sigma^2$ and support in $(0,\infty)$. Hence,  as long as $\underset{\sigma^2 \rightarrow 0}{\lim}\Gamma_1>0$, $\underset{\sigma^2 \rightarrow 0}{\lim}\P\left(A<\Gamma_1^2(1-erc({\bf X},\mathcal{J}))^2\right)>0$, which in turn imply that $\underset{\sigma^2 \rightarrow 0}{\lim}\P(\mathcal{E}_1)>0$. 

We next consider $\P(\mathcal{E}_2)$. Since $\mathcal{I}\subset\mathcal{J}$, we have ${\bf X}_{\mathcal{I}}\beta_{\mathcal{I}}={\bf X}_{\mathcal{J}}\beta_{\mathcal{J}}$ with appropriate zero entries in $\beta_{\mathcal{J}}$. Thus, ${\bf b}^{\mathcal{J}} \sim \mathcal{N}(\beta_{\mathcal{J}},\sigma^2({\bf X}_{\mathcal{J}}^T{\bf X}_{\mathcal{J}})^{-1})$. Hence, $|\mathcal{T}|=|\mathcal{J}|$ iff a false discovery is made in the thresholding procedure which gives $\mathcal{T}$. From Lemma \ref{l1penalty_necessary_lemma}, it follows that $\underset{\sigma^2 \rightarrow 0}{\lim}\P(\mathcal{E}_2)=\underset{\sigma^2 \rightarrow 0}{\lim}\P(|\mathcal{T}|=|\mathcal{J}|)>0$ as long as $\underset{\sigma^2 \rightarrow 0}{\lim}\Gamma_1<\infty$. 

A careful analysis of the events $\mathcal{E}_1$ and $\mathcal{E}_2$ reveals that $\mathcal{E}_1$ depends only on the component $({\bf I}_n-{\bf P}_{\mathcal{J}}){\bf w}$ of ${\bf w}$ and $\mathcal{E}_2$ depends only on the component ${\bf P}_{\mathcal{J}}{\bf w}$. Since, these two components are orthogonal and ${\bf w}$ is Gaussian, it follows that $\mathcal{E}_1$ and $\mathcal{E}_2$ are mutually independent, i.e., $\P(\mathcal{E}_1 \cap \mathcal{E}_2)=\P(\mathcal{E}_1)\P(\mathcal{E}_2)$.   Since,  $\underset{\sigma^2 \rightarrow 0}{\lim}\P(\mathcal{E}_1)>0$ and $\underset{\sigma^2 \rightarrow 0}{\lim}\P(\mathcal{E}_2)>0$, it follows that $\underset{\sigma^2 \rightarrow 0}{\lim}PE \geq \underset{\sigma^2 \rightarrow 0}{\lim}\P(\mathcal{E}_1\cap \mathcal{E}_1)=\underset{\sigma^2 \rightarrow 0}{\lim}\P(\mathcal{E}_1)\underset{\sigma^2 \rightarrow 0}{\lim}\P(\mathcal{E}_2)>0$, unless $\underset{\sigma^2 \rightarrow 0}{\lim}\Gamma_1=\infty$.
\end{proof}
It must be mentioned that  an index set $\mathcal{J}\supset\mathcal{I}$ satisfying ERC need not exist in all situations where $\mathcal{I}$ satisfy ERC. In that sense, Theorem \ref{l1penalty_necessary} is less general than Theorem \ref{l1_penalty_sufficient}. Nevertheless, Theorem \ref{l1penalty_necessary} is applicable in many practical settings. For example, if ${\bf X} \in \mathbb{R}^{n \times p}$  is orthonormal, then ${\bf X}$ satisfies ERC for all possible index sets $\mathcal{J} \subseteq [p]$. Similarly, if ${\bf X}$ satisfies the MIC of order $j$, i.e., $\mu_{\bf X} \leq \dfrac{1}{2j-1}$ and $j>k^*$, then ${\bf X}$ satisfies ERC for all $j>k^*$ sized index sets.  In both these situations, an index set (in fact many) $\mathcal{J}\supset \mathcal{I}$ satisfying ERC exists and $l_1$-penalty will be inconsistent without the required SNR adaptation. Theorem \ref{l1penalty_necessary} proves that the values of $\Gamma_1$ discussed in literature makes $l_1$-penalty high SNR inconsistent even in the simple case of orthonormal design matrix. Next we establish the necessity of $\underset{\sigma^2 \rightarrow 0}{\lim}\sigma\Gamma_1=0$ for high SNR consistency. 
\begin{thm}\label{l1_penalty_necessity2}
Suppose that the matrix  support pair $({\bf X},\mathcal{I})$ satisfy ERC and $k^*\geq 1$. Then, $l_1$-penalty  will be high SNR consistent only if $\underset{\sigma^2 \rightarrow 0}{\lim}\sigma\Gamma_1=0$. 
\end{thm}
\begin{proof}
Let $\mathcal{J}$ be any index set satisfying $\mathcal{J} \subset \mathcal{I}$. Since, $\mathcal{I}$ satisfy ERC, $\mathcal{J}$ will also satisfy ERC. Consider the event $\mathcal{E}:\{\|{\bf X}^T({\bf y}-{\bf y}^{\mathcal{J}})\|_{\infty}<\Gamma_1\sigma \left(1-erc({\bf X},\mathcal{J})\right)\}$, where ${\bf y}^{\mathcal{J}}={\bf P}_{\mathcal{J}}{\bf y}$. If $\mathcal{E}$ is true, then by Lemma \ref{l1_penalty_relax},  $\hat{\mathcal{I}}=supp(\hat{\beta})\subseteq \mathcal{J}\subset\mathcal{I}$. Thus, $PE\geq \P(\mathcal{E})$. The following bound on $\P(\mathcal{E})$ follows from Cauchy Schwartz inequality and the unit $l_2$ norm of ${\bf X}_j$.
\begin{equation}\label{l111}
\begin{array}{ll}
\P(\mathcal{E})&=\P\left(\underset{j}{\max}|{\bf X}_j^T ({\bf I}_n-{\bf P}_{\mathcal{J}}){\bf y}|<\sigma\Gamma_1 (1-erc({\bf X},\mathcal{J})\right) \\
&\geq P\left(\|{\bf I}_n-{\bf P}_{\mathcal{J}}){\bf y}\|_2<\sigma\Gamma_1 (1-erc({\bf X},\mathcal{J})\right) \\ &= \P\left(\sigma^2 A<\sigma^2 \Gamma_1^2(1-erc({\bf X},\mathcal{J}))^2\right),
\end{array}
\end{equation}
where $A=\dfrac{\|({\bf I}_n-{\bf P}_{\mathcal{J}}){\bf y}\|_2^2}{\sigma^2}\sim \chi^2_{n-|\mathcal{J}|}(\dfrac{\lambda}{\sigma^2})$ and $\lambda=\|({\bf I}_n-{\bf P}_{\mathcal{J}}){\bf X \beta}\|_2^2>0$. By Lemma \ref{chi2convergence}, $\sigma^2A \overset{P}{\rightarrow} \lambda$ as $\sigma^2 \rightarrow 0$. Hence, if $\underset{\sigma^2\rightarrow 0}{\lim}\sigma^2 \Gamma_1^2(1-erc({\bf X},\mathcal{J}))^2>\lambda$, then as shown in the proof of Theorem  \ref{sparknecessity2}, $\underset{\sigma^2\rightarrow 0}{\lim}\P\left(\sigma^2 A<\sigma^2 \Gamma_1^2(1-erc({\bf X},\mathcal{J}))^2\right)=1$. However, $\lambda$ is unknown. Thus, to satisfy  $\underset{\sigma^2\rightarrow 0}{\lim}\sigma^2 \Gamma_1^2(1-erc({\bf X},\mathcal{J}))^2<\lambda$, it is necessary that $\underset{\sigma^2 \rightarrow 0}{\lim} \sigma^2\Gamma_1^2=0$ which is equivalent to $\underset{\sigma^2 \rightarrow 0}{\lim} \sigma\Gamma_1=0$.
\end{proof}
{A widely used formulation of $l_1$-penalty is given by $\hat{\beta}=\underset{{\bf b} \in \mathbb{R}^p}{\arg\min}\dfrac{1}{2}\|{\bf y}-{\bf X}{\bf b}\|_2^2+\lambda\|{\bf b}\|_1$ which is equivalent to the  formulation in this article by setting $\lambda=\Gamma_1\sigma$. In this formulation, $l_1$-penalty is high SNR consistent if $\underset{\sigma^2 \rightarrow 0}{\lim}\dfrac{\lambda}{\sigma}=\infty$ and $\underset{\sigma^2 \rightarrow 0}{\lim}{\lambda} =0$. An interesting case is that of a fixed $\sigma$ independent $\lambda$ like $\lambda=0.1$.   This choice of $\lambda$ satisfy $\underset{\sigma^2 \rightarrow 0}{\lim}\dfrac{\lambda}{\sigma}=\infty$ which is a necessary condition for high SNR consistency. However, the satisfiability of the necessary condition in Theorem \ref{l1_penalty_necessity2} depends upon on the signal $\beta$ (Please see the proof of Theorem \ref{l1_penalty_necessity2}). Hence, when \textit{a priori} knowledge of $\beta$ is not available, a fixed regularization parameter is not advisable from the vantage point of high SNR consistency.}
 \subsection{High SNR consistency of $l_1$-error}
 We next discuss the high SNR behaviour of $\hat{\beta}=\underset{{\bf b} \in \mathbb{R}^p}{\arg\min} \|{\bf b}\|_1, \ \text{subject \ to} \  \|{\bf y}-{\bf X}{\bf b}\|_2 \leq \Gamma_2 \sigma$  and $\hat{\mathcal{I}}=supp(\hat{\beta})$. $l_1$-penalty is the Lagrangian of the constrained  optimization problem given by $l_1$-error. The performance of $l_1$-error is dictated by the choice of $\Gamma_2$. Commonly used choice of $\Gamma_2$ include $\Gamma_2=\sqrt{n+2\sqrt{2n}}$\cite{candes2006stable}, $\Gamma_2=\sqrt{n+2\sqrt{n\log(n)}}$
 \cite{cai_l1_minimization} etc.  A high SNR analysis of $l_1$-error in terms of variable selection properties is not available in open literature to the best of our knowledge. The following theorem states the sufficient conditions for $l_1$-error to be high SNR consistent.
 \begin{thm}\label{l1_error_thm}
 $l_1$-error is high SNR consistent for any matrix support pair $({\bf X},\mathcal{I})$ satisfying the ERC provided that the tuning parameter $\Gamma_2$ satisfies $\underset{\sigma^2 \rightarrow 0}{\lim}\Gamma_2=\infty$ and $\underset{\sigma^2 \rightarrow 0}{\lim}\sigma\Gamma_2=0$.
 \end{thm}
 \begin{proof}
 The proof of Theorem \ref{l1_error_thm} is based on the  result in [Theorem 14,\cite{tropp2006just}]   regarding the minimizers of $l_1$-error.
 \begin{lemma}\label{l1_error_lemma}
  Let $\mathcal{J}$ be any index set satisfying ERC. If ${\bf y}^{\mathcal{J}}={\bf P}_{\mathcal{J}}{\bf y}$ satisfies 
 \begin{equation}\label{equation_lemma_l1_error}
 \Gamma_2^2\sigma^2\geq \|{\bf y}-{\bf y}^{\mathcal{J}}\|_{2}^2+\dfrac{\|{\bf X}^T({\bf y}-{\bf y}^{\mathcal{J}})\|_{\infty}^2\|{\bf X}_{\mathcal{J}}^{\dagger}\|_{2,1}^2}{(1-erc({\bf X},\mathcal{J}))^2},
 \end{equation}
  then $\hat{\beta}$ satisfies the following. \\
A1). $supp(\hat{\beta}) \subseteq \mathcal{J}$. \\
A2). $\hat{\beta}$ is the unique minimizer of $l_1$-error. \\
A3). $\mathcal{T}=\{j : |{\bf b}^{\mathcal{J}}(j)|>\Gamma_2 \sigma\|{\bf X}_{\mathcal{J}}^{\dagger}\|_{2,2}\}\subseteq supp(\hat{\beta})$. \\
Here ${\bf b}^{\mathcal{J}}={\bf X}_{\mathcal{J}}^{\dagger}{\bf y}$ is the LS estimate of $\beta_{\mathcal{J}}$.
\end{lemma} 
{ In words, Lemma \ref{l1_error_lemma} states that if the residual between ${\bf y}$ and the LS fit of ${\bf y}$ using the columns in ${\bf X}_{\mathcal{J}}$ has sufficiently low  correlation with the columns in ${\bf X}$ and sufficiently low $l_2$ norm, then  the support of the solution to $l_1$-error will be contained in $\mathcal{J}$. Further, $l_1$-error does not miss indices that have sufficiently large values in the restricted LS estimate ${\bf b}^{\mathcal{J}}$.} By the hypothesis of Theorem \ref{l1_error_thm}, the true support $\mathcal{I}$ satisfies $erc({\bf X},\mathcal{I})<1$.   Thus, if the event ${\mathcal{E}}_1=\{(\ref{equation_lemma_l1_error}) \text{\ is satisfied}\}$, then $\hat{\mathcal{I}}=supp(\hat{\beta})\subseteq \mathcal{I}$.  If the event ${\mathcal{E}_2}=\{\forall j : |{\bf b}^{\mathcal{I}}(j)|>\Gamma_2 \sigma\|{\bf X}_{\mathcal{I}}^{\dagger}\|_{2,2}\}=\{|\mathcal{T}|=k^*\}$ is also true, then $supp(\hat{\beta})=\mathcal{I}$. Thus $\P(\hat{\mathcal{I}}=\mathcal{I})\geq \P(\mathcal{E}_1\cap \mathcal{E}_2)$. 

We first analyse the probability of the event $\mathcal{E}_1$. By Cauchy Schwartz inequality and the fact that $\|{\bf X}_j\|_2=1$, we have $|{\bf X}_j^T({\bf y}-{\bf y}^{\mathcal{I}})|\leq \|{\bf y}-{\bf y}^{\mathcal{I}}\|_2, \forall j$. Thus, $\|{\bf X}^T({\bf y}-{\bf y}^{\mathcal{I}}) \|_{\infty}^2\leq \|{\bf y}-{\bf y}^{\mathcal{I}}\|_2^2$. Hence, $\Gamma_2^2\sigma^2>\|{\bf y}-{\bf y}^{\mathcal{I}}\|_{2}^2 a_{\mathcal{I}}$, where $a_{\mathcal{I}}=\left(1+\dfrac{\|{\bf X}_{\mathcal{I}}^{\dagger}\|_{2,1}^2}{(1-erc({\bf X},\mathcal{I}))^2}\right)$ implies (\ref{equation_lemma_l1_error}). Thus,
\begin{equation}
\P(\mathcal{E}_1)\geq \P\left(\dfrac{\|{\bf y}-{\bf y}^{\mathcal{I}}\|_2^2}{\sigma^2}<\Gamma_2^2 a_{\mathcal{I}}^{-1} \right)
\end{equation}
Note that ${\bf y}-{\bf y}^{\mathcal{I}}=({\bf I}_n-{\bf P}_{\mathcal{I}}){\bf y}=({\bf I}_n-{\bf P}_{\mathcal{I}}){\bf w}$. Hence, $A=\dfrac{\|{\bf y}-{\bf y}^{\mathcal{I}}\|_2^2}{\sigma^2}\sim \chi^2_{n-k^*}$. Since, $\chi^2_{n-k^*}$ is a B.I.P R.V with $\sigma^2$ independent distribution, it follows that $\underset{\sigma^2 \rightarrow 0}{\lim}\P(A<\Gamma_2^2 a_{\mathcal{I}}^{-1})=1$ if $\underset{\sigma^2 \rightarrow 0}{\lim} \Gamma_2=\infty$. This implies that $\underset{\sigma^2 \rightarrow 0}{\lim}\P(\mathcal{E}_1)=1$. 

Next we consider the event $\mathcal{E}_2$. The index set $\mathcal{T}$ in Lemma \ref{l1_error_lemma} can be rewritten as $\mathcal{T}=\{j:|{\bf b}^{\mathcal{I}}(j)|>\sigma c_j \Gamma_1 d_j\}$, where $c_j=\sqrt{({\bf X}_{\mathcal{I}}^T{\bf X}_{\mathcal{I}})^{-1}_{j,j}}$ and $d_j=\dfrac{\|{\bf X}_{\mathcal{J}}^{\dagger}\|_{2,2}}{c_j}$. The event $\{|\mathcal{T}|=k^*\}$ happens iff there is no missed discovery in the thresholding procedure generating $\mathcal{T}$. Then it follows from $\underset{\sigma^2 \rightarrow 0}{\lim}\sigma\Gamma_1=0$ and L2) of Lemma \ref{l1penalty_necessary_lemma} that $\underset{\sigma^2 \rightarrow 0}{\lim}\P(\mathcal{E}_2)=\underset{\sigma^2 \rightarrow 0}{\lim}\P(|\mathcal{T}|=k^*)=1$.   Since, $\underset{\sigma^2 \rightarrow 0}{\lim}\P(\mathcal{E}_1)=1$ and $\underset{\sigma^2 \rightarrow 0}{\lim}\P(\mathcal{E}_2)=1$, it follows that $\underset{\sigma^2 \rightarrow 0}{\lim}\P(\hat{\mathcal{I}}=\mathcal{I})\geq\underset{\sigma^2 \rightarrow 0}{\lim} \P(\mathcal{E}_1\cap\mathcal{E}_2)=1$.
\end{proof}
The following theorem states that the SNR adaptations outlined in Theorem \ref{l1_error_thm} are necessary for high SNR consistency.
\begin{thm}\label{l1_error_necessity}
The following statements regarding the high SNR consistency of $l_1$-error are true.\\
1). Suppose  $\exists \mathcal{J}\supset  \mathcal{I}$ such that the matrix support pair $({\bf X},\mathcal{J})$   satisfy ERC. Then $l_1$-error is high SNR consistent only if $\underset{\sigma^2 \rightarrow 0}{\lim}\Gamma_2=\infty$. \\
2). Suppose that the matrix  support pair $({\bf X},\mathcal{I})$ satisfy ERC and $k^*\geq 1$. Then, $l_1$-error  will be high SNR consistent only if $\underset{\sigma^2 \rightarrow 0}{\lim}\sigma\Gamma_2=0$.  
\end{thm}
\begin{proof} Similar to Theorem \ref{l1penalty_necessary} and Theorem \ref{l1_penalty_necessity2}.
\end{proof}
Note that the values of $\Gamma_2$ discussed in literature do not satisfy the NSCs outlined in Theorems \ref{l1_error_thm} and \ref{l1_error_necessity}. Hence, $l_1$-error with these values of $\Gamma_2$ will be inconsistent at high SNR. 
\subsection{ Analysis of  Dantzig selector  based SSP}
Here, we discuss the high SNR behaviour of $\hat{\beta}$ given by  
$ \hat{\beta}=\underset{{\bf b} \in \mathbb{R}^p}{\arg\min} \|{\bf b}\|_1, \ \text{subject \ to} \  \|{\bf X}^T({\bf y}-{\bf X}{\bf b})\|_{\infty} \leq \Gamma_3\sigma.
 $
 and $\hat{\mathcal{I}}=supp(\hat{\beta})$.  The properties of DS is determined largely by the hyper parameter $\Gamma_3$. Commonly used values include $\Gamma_3=\sqrt{2\log(p)}$ \cite{candes2007dantzig}, $\Gamma_3=(\frac{3}{2}+\sqrt{2\log(p)})$ \cite{cai2010stable} etc.  No high SNR consistency results for DS is reported in open literature to the best of our knowledge. Next we state and prove the NSCs for the high SNR consistency of DS when ${\bf X}$ is orthonormal. 
 \begin{thm}\label{DS_necessity}
 For an  orthonormal design matrix ${\bf X}$, DS is high SNR consistent iff $\underset{\sigma^2 \rightarrow 0}{\lim}\Gamma_3=\infty$ and $\underset{\sigma^2 \rightarrow 0}{\lim}\sigma\Gamma_3<\underset{j \in \mathcal{I}}{\min}|\beta_j|$.
 \end{thm}
 \begin{proof}
 When ${\bf X}$ is orthonormal, the solution to DS is given by $\hat{\beta}_j=\left(|({\bf X}^T{\bf y})_j|-\Gamma_3\sigma\right)_+sign\left(({\bf X}^T{\bf y})_j\right), \forall j$. Note that $(x)_+=x$ if $x>0$ and $(x)_+=0$ if $x\leq 0$.  Thus $\hat{\mathcal{I}}$ is obtained by thresholding the  vector $|{\bf X}^T{y}|$ at level $\sigma \Gamma_3$. Since, ${\bf X}$ is orthonormal, we have ${\bf X}^T{\bf y}\sim \mathcal{N}(\beta,\sigma^2{\bf I}_p)$.  The proof now follows directly from Lemma \ref{l1penalty_necessary_lemma}.  
 \end{proof}
Note that no \textit{a priori} knowledge of $\beta$ is available. Hence, to achieve consistency, it is necessary that  $\underset{\sigma^2 \rightarrow  0}{\lim}\sigma\Gamma_3=0$. It follows directly from Theorem \ref{DS_necessity} that the values of $\Gamma_3$ discussed in literature are inconsistent for orthonormal matrices. This implies the inconsistency of these tuning parameters in regression classes based on $\mu_{\bf X}$ and ERC which includes orthonormal matrices too. We now make an  observation regarding the NSCs developed for $l_0$-penalty, $l_1$-error, $l_1$-penalty and DS.
\begin{remark}
{The SNR adaptations prescribed for high SNR consistency  have many similarities. Even though $l_0$-penalty requires $\Gamma_0\sigma^2 \rightarrow 0$ whereas other algorithms requires $\Gamma_i\sigma \rightarrow 0$,  
 the effective regularization parameter, i.e.,  $\lambda_0=\Gamma_0\sigma^2$ for $l_0$-penalty and $\lambda_i=\Gamma_i\sigma$ for other algorithms satisfies $\lambda_i \rightarrow 0$ as $\sigma^2 \rightarrow 0$. In the absence of noise (i.e $\sigma^2=0$) equality constrained optimization problems (\ref{maxsparse}) and (\ref{l1_noiseless}) will correctly recover the support of $\beta$ under spark and ERC assumptions respectively. Further, when the effective regularization parameter $\lambda_i \rightarrow 0$, $l_0$-penalty automatically reduces to (\ref{maxsparse}), whereas, $l_1$-penalty, $l_1$-error and DS reduces to  (\ref{l1_noiseless}).  Hence,  the condition $\lambda_i \rightarrow 0$ as $\sigma^2 \rightarrow 0$ is a natural choice to transition from the  formulations for noisy data to the equality constrained $l_0$ or $l_1$ minimization ideal for noiseless data. }
\end{remark}
\section{ Analysis of Orthogonal Matching Pursuit}
OMP\cite{OMP_wang,wang2012recovery,cai2011orthogonal} is one of most popular techniques in the class of greedy algorithms to solve CS problems. Unlike the CR techniques like $l_1$-penalty which has a computational complexity $O(np^2)$, OMP has a complexity of only $O(npk^*)$. Consequently, OMP is more easily scalable to large scale problems than CR techniques. Further, the performance guarantees for OMP are only slightly weaker compared to CR techniques. An algorithmic description of OMP is given below.
\begin{enumerate}[Step 1:]
  \item Initialize the residual ${\bf r}^0={\bf y}$. Support estimate ${\mathcal{J}^0}=\phi$. Iteration counter $i=1$;
  \item Find the column index most correlated with the current residual ${\bf r}^{i-1}$, i.e., ${t_i}=\underset{t \in [p]}{\arg\max}|{\bf X}_t^T{\bf r}^{i-1}|.$
  \item Update support estimate: ${\mathcal{J}^i}={\mathcal{J}^{i-1}}\cup {t_i}$.
  \item Update residual: ${\bf r}^{i}=({\bf I}_n-{\bf P}_{\mathcal{J}^i}){\bf y}$.
  \item Repeat Steps 2-4, if stopping condition (SC)  is not met, else, output $\hat{\mathcal{I}}=\mathcal{J}^i$.
\end{enumerate}
The properties of OMP is determined by the SC. A large body of literature regarding OMP assumes \textit{a priori} knowledge of sparsity level of $\beta$, i.e., $k^*$ and run $k^*$ iterations of OMP\cite{OMP_wang,wang2012recovery}. When $k^*$ is not known, two popular SCs for  OMP are discussed in literature. One SC called residual power based stopping condition (RPSC) terminate iterations when the residual power becomes too low (i.e., $\|{\bf r}^{i}\|_2<\sigma \Gamma_4$) and other SC called residual correlation based stopping condition (RCSC) terminate iterations when the maximum correlation of columns in ${\bf X}$ with the residual becomes too low (i.e., ${\|\bf X}^T{\bf r}^i\|_{\infty}<\sigma \Gamma_5$).  A commonly used value of $\Gamma_4$ is $\Gamma_4=\sqrt{n+2\sqrt{n\log(n)}}$ and that of $\Gamma_5$ is $\Gamma_5=\sqrt{2(1+\eta)\log(p)}$\cite{cai2011orthogonal}. Here $\eta>0$ is a constant. The following theorems state the sufficient conditions for  OMP with RPSC and RCSC to be high SNR consistent.
\begin{thm}\label{OMPres}
 OMP with  RPSC is high SNR consistent for any matrix ${\bf X}$ and signal $\beta$ satisfying the ERC provided that the hyper parameter $\Gamma_4$ satisfies $\underset{\sigma^2 \rightarrow 0}{\lim}\Gamma_4=\infty$ and $\underset{\sigma^2 \rightarrow 0}{\lim}\sigma\Gamma_4=0$.
\end{thm}
\begin{thm}\label{OMPcorr}
 OMP with  RCSC is high SNR consistent for any matrix ${\bf X}$ and signal $\beta$ satisfying the ERC provided that the hyper parameter $\Gamma_5$ satisfies $\underset{\sigma^2 \rightarrow 0}{\lim}\Gamma_5=\infty$ and $\underset{\sigma^2 \rightarrow 0}{\lim}\sigma\Gamma_5=0$.
\end{thm} 
\subsection{Proofs of Theorem \ref{OMPres} and Theorem \ref{OMPcorr}}
 Let us consider the two processes- OMP iterating without SC (P1) and verification of the SC (P2) separately. Specifically P1 returns a set of indexes in order, say $\{t_1,t_2,\dotsc\}$ and P2 returns a single index $j$ indicating where to stop. Then, the support estimate is given by $\hat{\mathcal{I}}=\{t_1,\dotsc,t_j\}$ and the indices after $j$, i.e., $\{t_{j+1},\dotsc\}$ will be discarded.  Let $\mathcal{E}_1$ denotes  the event $\{t_1,\dotsc,t_{k^*}\}=\mathcal{I}$, i.e., the first $k^*$ iterations of OMP returns all the $k^*$ indices in $\mathcal{I}$ and $\mathcal{E}_2$ denotes the event $\{\text{P2 returns } k^*\}$. Then $\P(\hat{\mathcal{I}}=\mathcal{I})=\P(\mathcal{E}_1\cap \mathcal{E}_2)$.

Let $N_i=\|{\bf X}^T({\bf I}_n-{\bf P}_{\mathcal{J}^{i-1}}){\bf w}\|_{\infty}$ denotes the maximum correlation between the columns in ${\bf X}$ and  noise component in the current residual ${\bf r}^{i-1}$ and $\beta_{min}=\underset{j \in \mathcal{I}}{\min}|\beta_{j}|$ denotes the minimum non-zero value in $\beta$. Then, using the analysis in Section \rom{5} of \cite{cai2011orthogonal}, $N_i<c_{\mathcal{I}}\beta_{min}$, where $c_\mathcal{I}=\dfrac{(1-erc({\bf X},\mathcal{I}))\lambda_{min}({\bf X}_{\mathcal{I}}^T{\bf X}_{\mathcal{I}})}{2\sqrt{k^*}}$ is a sufficient condition for selecting an index from $\mathcal{I}$ in the $i^{th}$ iteration ($\forall i\leq k^*$). Since, $\|{\bf X}_j\|_2=1$, it follows that $N_i\leq \|({\bf I}_n-{\bf P}_{\mathcal{J}^{i-1}}){\bf w}\|_2$.   Thus, $\P(\mathcal{E}_1)\geq \P(\underset{i=1\dotsc,k^*}{\cap}\{||({\bf I}_n-{\bf P}_{\mathcal{J}^{i-1}}){\bf w}||_2<c_{\mathcal{I}}\beta_{min}\})$. One can bound $\P(\mathcal{E}_1^C)$ using union bound and the inequality $\|({\bf I}_n-{\bf P}_{\mathcal{J}^{i-1}}){\bf w}\|_2\leq \|{\bf w}\|_2$ as 
\begin{equation}\label{exactk}
\begin{array}{ll}
\P(\mathcal{E}_1^C)&\leq \P(\underset{i=1\dotsc,k^*}{\cup}\{\|({\bf I}_n-{\bf P}_{\mathcal{J}^{i-1}}){\bf w}\|_2>c_{\mathcal{I}}\beta_{min}\})\\
&\leq \sum\limits_{i=1}^{k^*} \P(\dfrac{\|({\bf I}_n-{\bf P}_{\mathcal{J}^{i-1}}){\bf w}\|_2^2}{\sigma^2}>\dfrac{c_{\mathcal{I}}^2\beta_{min}^2}{\sigma^2}) \\
&\leq \sum\limits_{i=1}^{k^*} \P(Z>\dfrac{c_{\mathcal{I}}^2\beta_{min}^2}{\sigma^2})=k^*\P(Z>\dfrac{c_{\mathcal{I}}^2\beta_{min}^2}{\sigma^2}),
\end{array}
\end{equation}
where $Z=\dfrac{\|{\bf w}\|_2^2}{\sigma^2}\sim\chi^2_{n} $. Since, $Z$ is a B.I.P R.V with distribution independent of $\sigma^2$ and $\dfrac{c_{\mathcal{I}}^2\beta_{min}^2}{\sigma^2}\rightarrow \infty$ as $\sigma^2 \rightarrow 0$, we have $\underset{\sigma^2 \rightarrow 0}{\lim}\P(Z>\dfrac{c_{\mathcal{I}}^2\beta_{min}^2}{\sigma^2})=0$. This implies that $\underset{\sigma^2\rightarrow 0}{\lim}\P(\mathcal{E}_1)=1$. To summarize, if $erc({\bf X},\mathcal{I})<1$ and OMP runs exactly  $k^*$ iterations, then the true support can be detected exactly at high SNR. 

 The conditional probability $\P(\mathcal{E}_2|\mathcal{E}_1)$  is given by $\P(\mathcal{E}_2|\mathcal{E}_1)=\P(\{\text{SC is not satisified for}\ i=1,\dotsc,k^*-1 \} \cap \{\text{SC is satisfied for}\ i=k^*\}|\mathcal{E}_1)$. Complementing and applying union bound gives
 \begin{equation}\label{omp1}
 \begin{array}{ll}
 \P(\mathcal{E}_2^C|\mathcal{E}_1)\leq &\sum\limits_{i=1}^{k^*-1}\P(\overset{{S_i}}{\{\text{SC is satisfied for}\ i\}}|\mathcal{E}_1)+\\&\P(\overset{S_{k^*}} {\{\text{SC is not satisfied for }k^*}\}|\mathcal{E}_1).
 \end{array}
 \end{equation}
 
{\bf Proof of Theorem \ref{OMPres}:-}  For RPSC, the SC is given by $\{\|{\bf r}^i\|_2<\sigma\Gamma_4\}$. First consider $\P(S_i)=\P(\|{\bf r}^i\|_2<\sigma\Gamma_4)$ for $i<k^*$ in (\ref{omp1}). Using triangle inequality,  $\|{\bf r}^i\|_2\geq \|({\bf I}_n-{\bf P}_{\mathcal{J}^i}){\bf X}\beta\|_2 -\|({\bf I}_n-{\bf P}_{\mathcal{J}^i}){\bf w}\|_2$. Conditioned on $\mathcal{E}_1$, we have $\mathcal{J}^i\subset \mathcal{I}$ for $i<k^*$ and hence $\exists \lambda_i>0$ such that  $\|({\bf I}_n-{\bf P}_{\mathcal{J}^i}){\bf X}\beta\|_2>\lambda_i$, for all $\sigma^2>0$. Further,  $\|({\bf I}_n-{\bf P}_{\mathcal{J}^i}){\bf w}\|_2\leq \|{\bf w}\|_2$. Applying these bounds in $\P(S_i)=\P(\|{\bf r}^i\|_2<\sigma\Gamma_4)$ gives
\begin{equation}
\P(S_i)\leq \P(\|{\bf w}\|_2+\sigma\Gamma_4>\lambda_i),\ \forall i<k^*.
\end{equation}
Since, ${\bf w}\sim\mathcal{N}({\bf 0}_n,\sigma^2{\bf I}_n)$, we have $\|{\bf w}\|_2 \overset{P}{\rightarrow }0$ as $\sigma^2 \rightarrow 0$. By the hypothesis of Theorem 10, $\underset{\sigma^2 \rightarrow 0}{\lim}\sigma\Gamma_4=0$. Hence, $\|{\bf w}\|_2+\sigma\Gamma_4 \overset{P}{\rightarrow }0$ as $\sigma^2 \rightarrow 0$. Now by the definition of C.I.P, $\underset{\sigma^2 \rightarrow 0}{\lim}\P(\|{\bf w}\|_2+\sigma\Gamma_4>\lambda_i)=0$.
This implies that $\underset{\sigma^2 \rightarrow 0}{\lim}\P(S_i)=0, \forall i <k^*$.

 Next consider $\P(S_{k^*})$ in (\ref{omp1}). Conditioned on $\mathcal{E}_1$,  all the first $k^*$ iterations of OMP are correct, i.e., $\mathcal{J}^{k^*}=\mathcal{I}$. This implies that $\|{\bf r}^{k^*}\|_2^2=\|({\bf I}_n-{\bf P}_{\mathcal{I}}){\bf y}\|_2^2= \|({\bf I}_n-{\bf P}_{\mathcal{I}}){\bf w}\|_2^2\sim \sigma^2\chi^2_{n-k^*}$. Consequently, $\P(S_{k^*})=\P(\|{\bf r}^{k^*}\|_2^2>\sigma^2\Gamma_4^2)=\P(Z>\Gamma_4^2)$, where $Z=\dfrac{\|{\bf r}^{k^*}\|_2^2}{\sigma^2}\sim \chi^2_{n-k^*}$. Since $Z$ is a B.I.P R.V with distribution independent of $\sigma^2$ and $\Gamma_4 \rightarrow \infty$ as $\sigma^2 \rightarrow 0$, it follows that  $\underset{\sigma^2 \rightarrow 0}{\lim}\P(S_{k^*})=0$. Substituting $\underset{\sigma^2\rightarrow 0}{\lim}\P(S_i)=0$ for $i\leq k^*$ in (\ref{omp1}), we have  $\underset{\sigma^2 \rightarrow 0}{\lim}\P(\mathcal{E}_2|\mathcal{E}_1)=1$. Combining this with $\underset{\sigma^2 \rightarrow 0}{\lim} \P(\mathcal{E}_1)=1$ gives $\underset{\sigma^2 \rightarrow 0}{\lim}\P(\hat{\mathcal{I}}=\mathcal{I})=\underset{\sigma^2 \rightarrow 0}{\lim}\P({\mathcal{E}_1}\cap\mathcal{E}_2)=\underset{\sigma^2 \rightarrow 0}{\lim}\P({\mathcal{E}_1})\underset{\sigma^2 \rightarrow 0}{\lim}\P(\mathcal{E}_2|\mathcal{E}_1)=1$.\qed \\

{\bf Proof of Theorem \ref{OMPcorr}:-} For RCSC, the SC is given by $\{\|{\bf X}^T{\bf r}^i\|_{\infty}<\sigma\Gamma_5\}$. First consider $\P(S_i)=\P(\|{\bf X}^T{\bf r}^i\|_{\infty}<\sigma\Gamma_5)$ for $i<k^*$ in (\ref{omp1}).   $\|{\bf X}^T{\bf r}^i\|_{\infty}$ can be lower bounded using triangle inequality as $\|{\bf X}^T{\bf r}^i\|_{\infty}\geq \|{\bf X}^T({\bf I}_n-{\bf P}_{\mathcal{J}^{i-1}}){\bf X\beta}\|_{\infty}-\|{\bf X}^T({\bf I}_n-{\bf P}_{\mathcal{J}^{i-1}}){\bf w}\|_{\infty}.$ Further, $\|{\bf X}_i\|_2=1$ implies that $\|{\bf X}^T({\bf I}_n-{\bf P}_{\mathcal{J}^{i-1}}){\bf w}\|_{\infty} \leq \|({\bf I}_n-{\bf P}_{\mathcal{J}^{i-1}}){\bf w}\|_2\leq \|{\bf w}\|_2$. 
  Conditioned on $\mathcal{E}_1$, we have $\mathcal{J}^i\subset \mathcal{I}$ for $i<k^*$ and hence $\exists \lambda_i>0$ such that  $\|{\bf X}^T({\bf I}_n-{\bf P}_{\mathcal{J}^i}){\bf X\beta}\|_{\infty}>\lambda_i$, for all $\sigma^2>0$. Applying these bounds in $\P(S_i)=\P(\|{\bf X}^T{\bf r}^i\|_{\infty}<\sigma\Gamma_5)$ gives
\begin{equation}
\P(S_i)\leq \P(\|{\bf w}\|_2+\sigma\Gamma_5>\lambda_i),\ \forall i<k^*.
\end{equation}
Following the same arguments used in the proof of Theorem 10 we have $\underset{\sigma^2 \rightarrow 0}{\lim}\P(S_i)=0, \forall i <k^*$.

 Next we consider $\P(S_{k^*})=\P(\|{\bf X}^T{\bf r}^{k^*}\|_{\infty}>\sigma\Gamma_5)$. Since, the first $k^*$ iterations are correct, i.e., $\mathcal{J}^{k^*}=\mathcal{I}$, we have $\|{\bf X}^T{\bf r}^{k^*}\|_{\infty}=\|{\bf X}^T({\bf I}_n-{\bf P}_{\mathcal{I}}){\bf y}\|_{\infty}=\|{\bf X}^T({\bf I}_n-{\bf P}_{\mathcal{I}}){\bf w}\|_{\infty}$. Using Cauchy Schwartz inequality and $\|{\bf X}_j\|_2=1$, it follows that $\|{\bf X}^T({\bf I}_n-{\bf P}_{\mathcal{I}}){\bf w}\|_{\infty} \leq \|({\bf I}_n-{\bf P}_{\mathcal{I}}){\bf w}\|_{2}$. Hence, $\P(S_{k^*})\leq \P(\dfrac{\|({\bf I}_n-{\bf P}_{\mathcal{I}}){\bf w}\|_{2}^2}{\sigma^2}>\Gamma_5^2)$.  Since, $\dfrac{\|({\bf I}_n-{\bf P}_{\mathcal{I}}){\bf w}\|_2^2}{\sigma^2}\sim \chi^2_{n-k^*}$ is a B.I.P R.V and the term $\Gamma_5^2\rightarrow \infty$, it follows that $\underset{\sigma^2 \rightarrow 0}{\lim}\P(S_{k^*})=0$. Substituting $\underset{\sigma^2 \rightarrow 0}{\lim}\P(S_i)=0$ for $i\leq k^*$  in (\ref{omp1}), we have  $\underset{\sigma^2 \rightarrow 0}{\lim}\P(\mathcal{E}_2|\mathcal{E}_1)=1$. Combining this with $\underset{\sigma^2 \rightarrow 0}{\lim} \P(\mathcal{E}_1)=1$ gives $\underset{\sigma^2 \rightarrow 0}{\lim}\P(\hat{\mathcal{I}}=\mathcal{I})=\underset{\sigma^2 \rightarrow 0}{\lim}\P({\mathcal{E}_1}\cap\mathcal{E}_2)=\underset{\sigma^2 \rightarrow 0}{\lim}\P({\mathcal{E}_1})\underset{\sigma^2 \rightarrow 0}{\lim}\P(\mathcal{E}_2|\mathcal{E}_1)=1$. \qed
 
\begin{remark} The following observations can be made about the convergence rates in RPSC and RCSC. The rate at which $\P(\mathcal{E}_1)$ converges to one is independent of $\Gamma_4$ or $\Gamma_5$. First consider $\P(S_i)$ for $i<k^*$ and let $\Gamma_4=\dfrac{1}{\sigma^{\alpha}}$ be the SNR adaptation. Then the rate at which $\P(S_i)$ converges to zero is maximum when $\alpha=0$ and decreases with increasing $\alpha$. However, the rate at which $\P(S_{k^*})$ converges to zero increases with increasing $\alpha$. 
\end{remark}
 
 \section{Numerical Simulations}
 Here we numerically verify the results proved in Theorems 1-11.  We consider two classes of matrices for simulations.  \\
{  {\bf  ERC matrix}: We consider a $n \times 2n$ matrix ${\bf X}$ formed by the concatenation of a $n \times n$ identity matrix and a Hadamard matrix  of size $n \times n$ denoted by ${\bf H}_n$, i.e., ${\bf X}=[{\bf I}_n,{\bf H}_n.]$. It is well known that this matrix has mutual coherence $\mu_{\bf X}=\dfrac{1}{\sqrt{n}}$[Chapter 2,\cite{elad2010sparse}]. We fix $n$ as $n=32$ and for this value of $n$, ${\bf X}$ satisfy MIC for any $\beta$ with sparsity $k^*\leq \dfrac{1}{2}(1+\sqrt{n})=3.3284$. As explained in section \rom{2}, MIC implies ERC  also.} \\
{ {\bf Random matrix:} A random matrix ${\bf X}$  is generated using \textit{i.i.d} ${\bf X}_{i,j}\sim \mathcal{N}(0,1)$ R.Vs and  columns in this matrix are later normalised to have unit $l_2$ norm.   In each iteration the matrix ${\bf X}$ is independently generated. The matrix support pair thus generated in each iteration may or may not satisfy ERC.\\
All non zero entries have same magnitude (denoted by $\beta_k$ in  figures) but random signs. Further, the $k^*$ non zero entries are selected randomly from the set $[p]$. The figures are produced after performing $10^5$ iterations at each SNR level.
\subsection{Performance of $l_0$-penalty.}
 \begin{figure}
\includegraphics[width=\columnwidth,height=70mm]{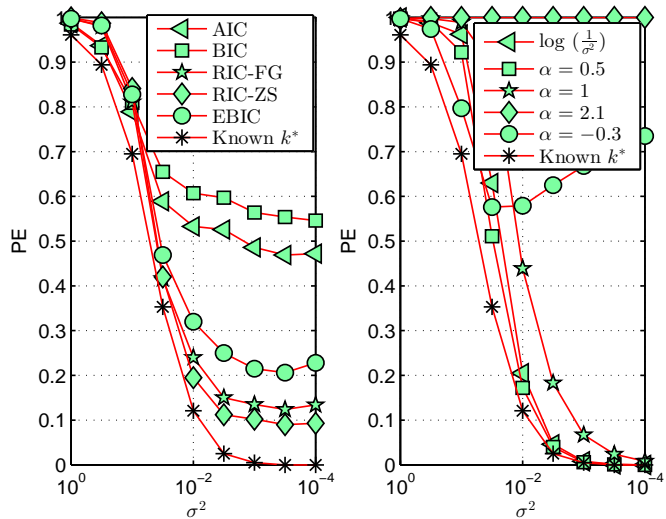}
\caption{Performance of $l_0$-penalty with a $5\times10$ random matrix. $\beta_k=\pm 1$ and $k^*=2$. }
\label{fig_l0}
\end{figure}
The performance of $l_0$-penalty with different values of $\Gamma_0$ is reported in Fig.\ref{fig_l0}. The matrix under consideration is a $5 \times 10$ random matrix.  ``Known $k^*$" represents the performance of an oracle SSP with \textit{a priori} information of $k^*$. This SSP estimates $\hat{\mathcal{I}}$ using $\hat{\mathcal{I}}=\underset{{\mathcal{J}\subset[p]},{|\mathcal{J}|=k^*}}{\arg\min}|| ({\bf I}_n-{\bf P}_{\mathcal{J}}){\bf y}||_2^2$  and will have superior performance when compared with $l_0$-penalty which is oblivious to $k^*$.

L.H.S of Fig.\ref{fig_l0} gives the performance of $l_0$-penalty with SNR independent values of $\Gamma_0$ discussed in literature. AIC uses $\Gamma_0 =2$, BIC uses $\Gamma_0 =\log(n)$, RIC-FG uses $\Gamma_0 =2\log(p)$\cite{RIC-FG}, RIC-ZS uses $\Gamma_0 =2\log(p)+2\log(\log(p))$ \cite{RIC-ZS}  and EBIC uses $\Gamma_0=\log(n)+\dfrac{2}{||{\bf b}||_0}\log({p \choose ||{\bf b}||_0})$. As predicted by Theorem \ref{sparknecessity1}, $l_0$-penalty with all these values of $\Gamma_0$ are inconsistent at high SNR. The performance of RIC-ZS is the best among the values of $\Gamma_0$ under consideration. The performance of BIC and AIC are much poorer compared to other schemes. When $n=5$, $\Gamma_0=2$ in AIC is bigger than $\Gamma_0=\log(n)$ of BIC and this explains the inferior performance of BIC \textit{viz a viz} AIC. For higher values of $n$, BIC will perform better than AIC.

R.H.S gives the performance of $l_0$-penalty with $\Gamma_0=f(\sigma^2)[\log(n)+\dfrac{2}{\|{\bf b}\|_0}\log({p \choose \|{\bf b}\|_0})]$, i.e., a SNR adaptation is added to  EBIC penalty. $``\log(\dfrac{1}{\sigma^2})''$ in Fig.\ref{fig_l0} represents  $f(\sigma^2)= \log(\dfrac{1}{\sigma^2})$. This SNR adaptation satisfies the conditions in Theorem \ref{Sparktheorem} and is common in popular MOS criteria like NMDL, g-MDL etc. \cite{tsp}. The schemes represented using $\alpha=(.)$ has $f(\sigma^2)=\dfrac{1}{\sigma^{\alpha}}$. Among the values of $\alpha$ considered in Fig.\ref{fig_l0}, $\alpha=0.5$ and $\alpha=1$ satisfies the conditions in Theorem \ref{Sparktheorem},  $\alpha=-0.3$ violates Theorem \ref{sparknecessity1} and $\alpha=2.1$ violates Theorem \ref{sparknecessity2} respectively. As predicted by Theorems 1-3, only $``\log(\dfrac{1}{\sigma^2})''$, $\alpha=0.5$ and $\alpha=1$ that satisfies the conditions in Theorem \ref{Sparktheorem} are high SNR consistent. This verify the NSCs derived in section \rom{3}. Further, the performance of $l_0$-penalty with $\Gamma_0$ represented by ``$\log(\dfrac{1}{\sigma^2})$'' and  $\alpha=0.5$  are very close to the optimal scheme represented by ``Known $k^*$" across the entire SNR range. This suggest the finite SNR utility of the SNR adaptations suggested by Theorem \ref{Sparktheorem}. 
\subsection{Performance of $l_1$-penalty and $l_1$-error at high SNR.}
L.H.S of Fig.\ref{fig_l1_ERC} gives the performance of  $l_1$-penalty and R.H.S of Fig.\ref{fig_l1_ERC} gives the performance of  $l_1$-error respectively. Both these SSPs are evaluated for the $32 \times 64$  ERC matrix previously defined  and a  $75 \times 100$ random matrix. ``$2\sqrt{2\log(p)}$" in L.H.S represents the performance of $l_1$-penalty with $\Gamma_1=2\sqrt{2\log(p)}$  \cite{candes2009near} and ``$\alpha=(.)$`` represents $l_1$-penalty with $\Gamma_1=\dfrac{1}{\sigma^{\alpha}}2\sqrt{2\log(p)}$.  Similarly, in the R.H.S, ``$\sqrt{n+2\sqrt{2n}}$" represents the $l_1$-error with $\Gamma_2=\sqrt{n+2\sqrt{2n}}$ \cite{candes2006stable} and ``$\alpha=(.)$" represents $l_1$-error with $\Gamma_2=\dfrac{1}{\sigma^{\alpha}}\sqrt{n+2\sqrt{2n}}$. In both cases, $\alpha=(.)$ incorporates a SNR adaptation into a well known value of $\Gamma_1$ and $\Gamma_2$.  By Theorems 4-8, these SNR adaptations are  consistent iff $0<\alpha<1$.

 First we consider the performance of $l_1$-penalty  for the matrix ${\bf X}$ satisfying ERC. It is clear from Fig.\ref{fig_l1_ERC} that $l_1$-penalty with $\Gamma_1=2\sqrt{2\log(p)}$ floors at high SNR with a $PE\approx 10^{-2.5}$. Hence, $l_1$-penalty with $\Gamma_1=2\sqrt{2\log(p)}$ is inconsistent at high SNR and this validates Theorem \ref{l1penalty_necessary}. Other $\sigma^2$ independent values of $\Gamma_1$ discussed in Section \rom{4} also floors at high SNR. On the contrary, $l_1$-penalty with SNR dependent $\Gamma_1$  does not floor at high SNR and this validates Theorem \ref{l1_penalty_sufficient}. Further, $\Gamma_1$ with $\alpha=0.1$ performs better than $\Gamma_1=2\sqrt{2\log(p)}$ even for $\sigma^2 \approx 0.01$. In the same setting, $l_1$-error with $\Gamma_2=\sqrt{n+2\sqrt{2n}}$ is inconsistent at high SNR. In fact PE for $\Gamma_2=\sqrt{n+2\sqrt{2n}}$ floors at $PE\approx 10^{-1.25}$ at high SNR. It is evident from Fig.\ref{fig_l1_ERC} that $\Gamma_2=\dfrac{1}{\sigma^{\alpha}}\sqrt{n+2\sqrt{2n}}$, where $\alpha=0.15$ and $\alpha=0.3$  are high SNR consistent. These results validates Theorems 7-8. In fact $l_1$-error with $\alpha=0.15$ and $\alpha=0.3$ performs much better than the SNR independent $\Gamma_2=\sqrt{n+2\sqrt{2n}}$ from $\sigma^2 \approx 0.01$ onwards. 
 
\begin{figure}
\includegraphics[width=\columnwidth,height=70mm]{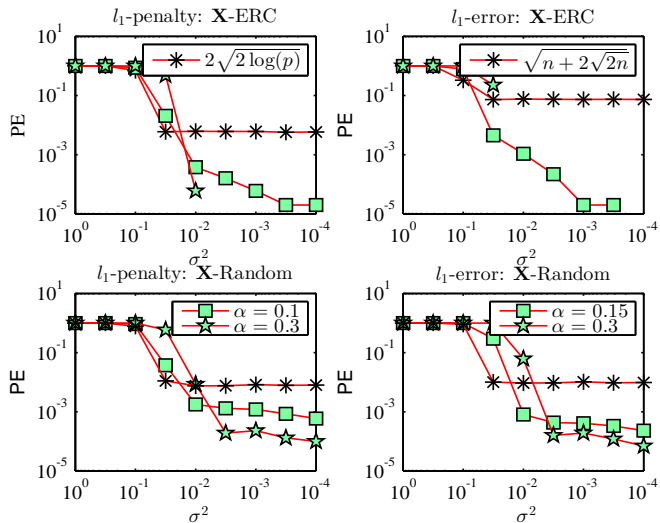}
\caption{Performance of $l_1$-penalty and $l_1$-error for a $32 \times 64$ ERC matrix and $75 \times 100$ random matrix. $k^*=3$ and $\beta_k=\pm 1$.  }
\label{fig_l1_ERC}
\end{figure}

{Next we consider the performance of $l_1$-penalty and $l_1$-error when ${\bf X}$ is a random $75 \times 100$ matrix. Here also  $l_1$-penalty and $l_1$-error with values of $\Gamma_1$ and $\Gamma_2$  independent of $\sigma^2$ floors at high SNR. However, unlike the case of ERC matrix,   $\Gamma_1$ and $\Gamma_2$ with SNR adaptations stipulated by Theorem \ref{l1_penalty_sufficient} and Theorem \ref{l1_error_thm} appears to floor at high SNR. This is because of the fact that there is a non zero probability $p_{erc}>0$ with which a particular realization of $({\bf X},\beta)$ pair fails to satisfy conditions like ERC. In fact  $p_{erc}$ decreases exponentially with increasing $n$. Hence, for random matrices $p_{erc}$ dictates the PE at which $l_1$-penalty and $l_1$-error floors. Note that the level at which $PE$ of $l_1$-penalty with $\Gamma_1$ and $\Gamma_2$ satisfying the SNR adaptations stipulated by Theorem \ref{l1_penalty_sufficient} and Theorem \ref{l1_error_thm} floors is significantly lower than the case with SNR independent $\Gamma_1$ and $\Gamma_2$.  This indicates that the proposed SNR adaptations can improve performance in situations beyond the regression classes for which high SNR consistency is established.   }
\subsection{Performance of OMP at high SNR.}
 L.H.S of Fig.\ref{fig_OMP_ERC} presents the performance of OMP with RPSC and R.H.S presents the performance of OMP with RCSC respectively. Both these SSPs are evaluated for the ERC matrix previously defined  and a  $75 \times 100$ random matrix.  ``Known $k^*$" represents a hypothetical SSP which runs OMP for exactly $k^*=3$ iterations. ``$f_n$" in the L.H.S  represents the performance of RPSC with $\Gamma_4= \sqrt{n+2\sqrt{n\log(n)}}$  and ``$\alpha=(.)$" represents the performance of RPSC with $\Gamma_4=\dfrac{1}{\sigma^{\alpha}}\sqrt{n+2\sqrt{n\log(n)}}$. Similarly, ``$f_p$" in the R.H.S  represents the performance of RCSC with $\Gamma_5=\sqrt{4\log(p)}$  and ``$\alpha=(.)$" represents the performance of RCSC with $\Gamma_5=\dfrac{1}{\sigma^{\alpha}}\sqrt{4\log(p)}$. $\Gamma_4= \sqrt{n+2\sqrt{n\log(n)}}$ and $\Gamma_5=\sqrt{4\log(p)}$ are suggested in \cite{cai2011orthogonal}. ``$\alpha=(.)$" in both cases incorporate a SNR adaptation into these well known stopping parameters.   It is clear from the Fig.\ref{fig_OMP_ERC} that OMP with SC independent of $\sigma^2$ floors at high SNR for both ERC and random matrices, whereas, the  flooring of PE is not present in OMP with SC satisfying Theorems \ref{OMPres} and  \ref{OMPcorr} for ERC matrix. For random matrix, the performance of OMP with proposed SNR adaptations floors at a PE level equal to that of OMP with known $k^*$. This flooring is also due to the causes explained for $l_1$-penalty and $l_1$-error.   
\begin{figure}
\includegraphics[width=\columnwidth,height=70mm]{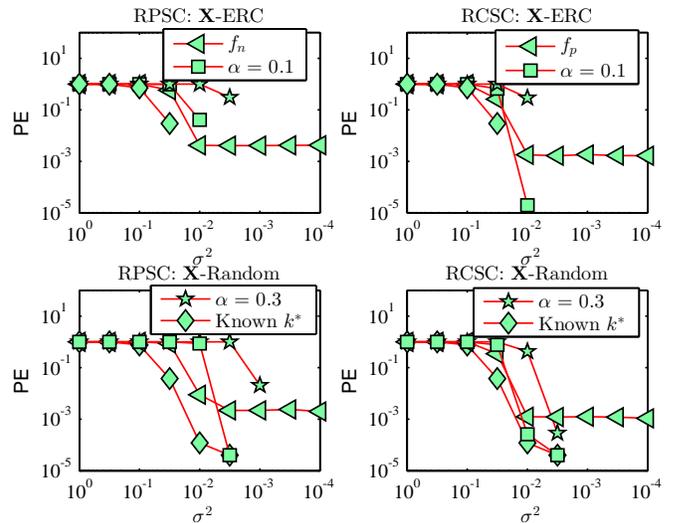}
\caption{Performance of OMP  with  RPSC and RCSC for a $32 \times 64$ ERC matrix and $75 \times 100$ random matrix. $k^*=3$ and $\beta_k=\pm 1$. }
\label{fig_OMP_ERC}
\end{figure}
\subsection { On the choice of SNR adaptations.}
 Fig.\ref{fig_conv} presents the effect of signal strength $|\beta_j|$ and SNR adaptations on the convergence rates of $l_1$-penalty and OMP-RPSC. ``$f_n$" represents RPSC with $\Gamma_4=\sqrt{n+2\sqrt{2\log(n)}}$ as before. ``$\alpha=(.)$" represents $l_1$-penalty with $\Gamma_1=\dfrac{1}{\sigma^{\alpha}}2\sqrt{2\log(p)}$ and RPSC with $\Gamma_4=\dfrac{1}{\sigma^{\alpha}}\sqrt{n+2\sqrt{2\log(n)}}$. By Theorems \ref{l1_penalty_sufficient} and \ref{OMPres}, the SNR adaptations represented by $\alpha=(.)$ will be consistent for both $l_1$-penalty and RPSC iff $0<\alpha<1$. However, the deviations from the base tuning parameters (i.e., $2\sqrt{2\log(p)}$ and $\sqrt{n+2\sqrt{2\log(n)}}$)   will be more pronounced as $\alpha$ increases. This will influence the rate at which $PE$ converges to zero.

At very high SNR, the performance of $l_1$-penalty and OMP-RPSC  with larger values of $\alpha$ will be better. This is true for both low and high values of regression coefficients (i.e., $\beta_j=0.5$ and $\beta_j=3$).  Throughout the moderate to high SNR range, the performance of these  algorithms with high values of $\alpha$ will be poor in comparison with the base tuning parameter when $|\beta_j|$ is low. In the same SNR and signal strength regime the performance  with low values of $\alpha$ will be better than both base tuning parameter and high value of $\alpha$.  As the signal strength improves, the performance of these algorithms improves for all values of $\alpha$.  However, the performance with high values of $\alpha$ will be much better than the performance with low values of $\alpha$ when $|\beta_j|$ is high. Note that the PE  with base tuning parameter floors at the same value irrespective of signal strength. The numerical results are in line with the inferences derived from the convergence rate analysis of $l_1$-penalty. Similar inferences can be derived from the numerical experiments (not shown) conducted for other CS algorithms considered in this paper.

Note that the very high SNR regime is rarely encountered in practice. Further, a low value of $\alpha$ will  provide a performance atleast as good as the performance of the base parameter in the moderate SNR range irrespective of the signal strength and a progressively improving performance as the SNR or the signal strength improves. Hence, by following the philosophy of minimizing the worst case risk, it will be advisable to choose smaller values of $\alpha$ like $\alpha=0.1$ for practical applications. 
\begin{figure}[h!]
\includegraphics[width=\columnwidth,height=70mm]{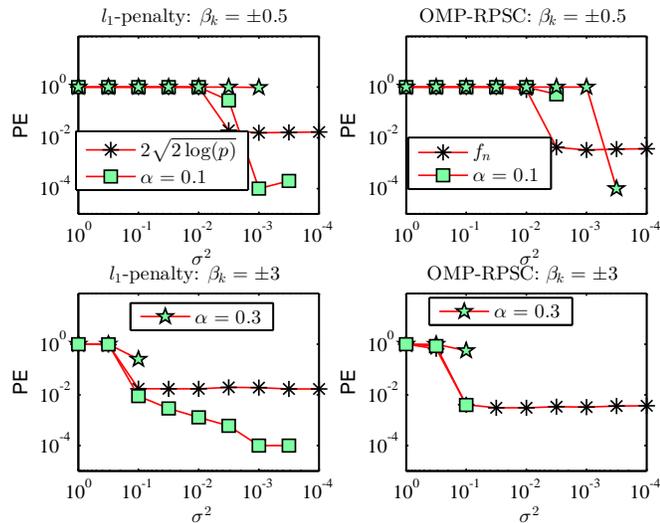}
\caption{Convergence rates for OMP with RPSC and $l_1$-penalty for a $32 \times 64$ ERC matrix and $k^*=3$.  }
\label{fig_conv}
\end{figure}
\section{Conclusion}
NSCs for the high SNR consistency of CS algorithms like $l_0$-penalty, $l_1$-penalty, $l_1$-error, DS and OMP are derived in this paper. Aforementioned algorithms with the tuning parameters discussed in literature are analytically and numerically shown to be inconsistent at high SNR.  Novel tuning parameters for these CS algorithms are derived based on the sufficient conditions and justified using convergence rate analysis. CS algorithms with the proposed tuning parameters are numerically shown to perform  better than existing tuning parameters.


 \bibliographystyle{IEEEtran}
\bibliography{compressivesensing.bib}

\begin{thebibliography}{10}
\providecommand{\url}[1]{#1}
\csname url@samestyle\endcsname
\providecommand{\newblock}{\relax}
\providecommand{\bibinfo}[2]{#2}
\providecommand{\BIBentrySTDinterwordspacing}{\spaceskip=0pt\relax}
\providecommand{\BIBentryALTinterwordstretchfactor}{4}
\providecommand{\BIBentryALTinterwordspacing}{\spaceskip=\fontdimen2\font plus
\BIBentryALTinterwordstretchfactor\fontdimen3\font minus
  \fontdimen4\font\relax}
\providecommand{\BIBforeignlanguage}[2]{{%
\expandafter\ifx\csname l@#1\endcsname\relax
\typeout{** WARNING: IEEEtran.bst: No hyphenation pattern has been}%
\typeout{** loaded for the language `#1'. Using the pattern for}%
\typeout{** the default language instead.}%
\else
\language=\csname l@#1\endcsname
\fi
#2}}
\providecommand{\BIBdecl}{\relax}
\BIBdecl

\bibitem{eldar2012compressed}
Y.~C. Eldar and G.~Kutyniok, \emph{Compressed sensing: {T}heory and
  applications}.\hskip 1em plus 0.5em minus 0.4em\relax Cambridge University
  Press, 2012.

\bibitem{candes2007dantzig}
T.~T. Emmanuel~Candes, ``The {D}antzig selector: {S}tatistical estimation when
  p is much larger than n,'' \emph{Ann. Stat.}, vol.~35, no.~6, pp. 2313--2351,
  2007.

\bibitem{tropp2004greed}
J.~A. Tropp, ``Greed is good: Algorithmic results for sparse approximation,''
  \emph{IEEE Trans. Inf. Theory}, vol.~50, no.~10, pp. 2231--2242, 2004.

\bibitem{tropp2006just}
J.~Tropp, ``Just relax: {C}onvex programming methods for identifying sparse
  signals in noise,'' \emph{IEEE Trans. Inf. Theory}, vol.~52, no.~3, pp.
  1030--1051, March 2006.

\bibitem{wipf2004sparse}
D.~P. Wipf and B.~D. Rao, ``Sparse {B}ayesian learning for basis selection,''
  \emph{IEEE Trans. Signal Process.}, vol.~52, no.~8, pp. 2153--2164, 2004.

\bibitem{masood2013sparse}
M.~Masood and T.~Y. Al-Naffouri, ``Sparse reconstruction using distribution
  agnostic {B}ayesian matching pursuit,'' \emph{IEEE Trans. Signal Process.},
  vol.~61, no.~21, pp. 5298--5309, 2013.

\bibitem{ben2010coherence}
Z.~Ben-Haim, Y.~C. Eldar, and M.~Elad, ``Coherence-based performance guarantees
  for estimating a sparse vector under random noise,'' \emph{IEEE Trans. Signal
  Process.}, vol.~58, no.~10, pp. 5030--5043, 2010.

\bibitem{spl}
K.~Sreejith and S.~Kalyani, ``High {SNR} consistent thresholding for variable
  selection,'' \emph{IEEE Signal Process. Lett.}, vol.~22, no.~11, pp.
  1940--1944, Nov 2015.

\bibitem{zhao2006model}
P.~Zhao and B.~Yu, ``On model selection consistency of {LASSO},'' \emph{J.
  Mach. Learn. Res.}, vol.~7, pp. 2541--2563, 2006.

\bibitem{ding2011inconsistency}
Q.~Ding and S.~Kay, ``Inconsistency of the {MDL}: {O}n the performance of model
  order selection criteria with increasing signal-to-noise ratio,'' \emph{IEEE
  Trans. Signal Process.}, vol.~59, no.~5, pp. 1959--1969, May 2011.

\bibitem{kay2005exponentially}
S.~Kay, ``Exponentially embedded families - {N}ew approaches to model order
  estimation,'' \emph{IEEE Trans. Aerosp. Electron. Syst.}, vol.~41, no.~1, pp.
  333--345, Jan 2005.

\bibitem{rissanen2000mdl}
J.~Rissanen, ``{MDL} denoising,'' \emph{IEEE Trans. Inf. Theory}, vol.~46,
  no.~7, pp. 2537--2543, Nov 2000.

\bibitem{hansen2001model}
M.~H. Hansen and B.~Yu, ``Model selection and the principle of minimum
  description length,'' \emph{J. Amer. Statist. Assoc.}, vol.~96, no. 454, pp.
  746--774, 2001.

\bibitem{stoica2012proper}
P.~Stoica and P.~Babu, ``On the proper forms of {BIC} for model order
  selection,'' \emph{IEEE Trans. Signal Process.}, vol.~60, no.~9, pp.
  4956--4961, Sept 2012.

\bibitem{bayesian}
J.~Nielsen, M.~Christensen, and S.~Jensen, ``Bayesian model comparison and the
  {BIC} for regression models,'' in \emph{ICASSP}, May 2013, pp. 6362--6366.

\bibitem{SNLS}
J.~Rissanen, T.~Roos, and P.~Myllymäki, ``Model selection by sequentially
  normalized least squares,'' \emph{J. Multivariate Anal.}, vol. 101, no.~4,
  pp. 839 -- 849, 2010.

\bibitem{tsp}
S.~Kallummil and S.~Kalyani, ``High {SNR} consistent linear model order
  selection and subset selection,'' \emph{IEEE Trans. Signal Process.},
  vol.~64, no.~16, pp. 4307--4322, Aug 2016.

\bibitem{schmidt2012consistency}
D.~Schmidt and E.~Makalic, ``The consistency of {MDL} for linear regression
  models with increasing signal-to-noise ratio,'' \emph{IEEE Trans. Signal
  Process.}, vol.~60, no.~3, pp. 1508--1510, March 2012.

\bibitem{SNLShighSNR}
J.~M{\"a}{\"a}tt{\"a}, D.~F. Schmidt, and T.~Roos, ``Subset selection in linear
  regression using sequentially normalized least squares: {A}symptotic
  theory,'' \emph{SCAND. J. STAT.}, 2015.

\bibitem{foucart2013mathematical}
S.~Foucart and H.~Rauhut, \emph{A mathematical introduction to compressive
  sensing}.\hskip 1em plus 0.5em minus 0.4em\relax Springer, 2013.

\bibitem{OMP_wang}
J.~Wang, ``Support recovery with orthogonal matching pursuit in the presence of
  noise,'' \emph{IEEE Trans. Signal Process.}, vol.~63, no.~21, pp. 5868--5877,
  Nov 2015.

\bibitem{wang2012recovery}
J.~Wang and B.~Shim, ``On the recovery limit of sparse signals using orthogonal
  matching pursuit,'' \emph{IEEE Trans. Signal Process.}, vol.~60, no.~9, pp.
  4973--4976, 2012.

\bibitem{cai2011orthogonal}
T.~Cai and L.~Wang, ``Orthogonal matching pursuit for sparse signal recovery
  with noise,'' \emph{IEEE Trans. Inf. Theory}, vol.~57, no.~7, pp. 4680--4688,
  July 2011.

\bibitem{EEFPDF}
S.~Kay, Q.~Ding, B.~Tang, and H.~He, ``Probability density function estimation
  using the {EEF} with application to subset/feature selection,'' \emph{IEEE
  Trans. Signal Process.}, vol.~64, no.~3, pp. 641--651, Feb 2016.

\bibitem{multiuserCS}
B.~Shim and B.~Song, ``Multiuser detection via compressive sensing,''
  \emph{IEEE Commun. Lett.}, vol.~16, no.~7, pp. 972--974, July 2012.

\bibitem{fletcher2009off}
A.~K. Fletcher, S.~Rangan, and V.~K. Goyal, ``On-off random access channels:
  {A} compressed sensing framework,'' \emph{arXiv preprint arXiv:0903.1022},
  2009.

\bibitem{gerstoft2015multiple}
P.~Gerstoft, A.~Xenaki, and C.~F. Mecklenbr{\"a}uker, ``Multiple and single
  snapshot compressive beamforming,'' \emph{The Journal of the Acoustical
  Society of America}, vol. 138, no.~4, pp. 2003--2014, 2015.

\bibitem{chung2001course}
K.~L. Chung, \emph{A course in probability theory}.\hskip 1em plus 0.5em minus
  0.4em\relax Academic press, 2001.

\bibitem{RIC-FG}
D.~P. Foster and E.~I. George, ``The risk inflation criterion for multiple
  regression,'' \emph{Ann. Stat.}, pp. 1947--1975, 1994.

\bibitem{RIC-ZS}
Y.~Zhang and X.~Shen, ``Model selection procedure for high-dimensional data,''
  \emph{Stat. Anal. Data Min.}, vol.~3, no.~5, pp. 350--358, 2010.

\bibitem{E-BIC}
J.~Chen and Z.~Chen, ``Extended {B}ayesian information criteria for model
  selection with large model spaces,'' \emph{Biometrika}, vol.~95, no.~3, pp.
  759--771, 2008.

\bibitem{GICconsistent}
Y.~Kim, S.~Kwon, and H.~Choi, ``Consistent model selection criteria on high
  dimensions,'' \emph{J. Mach. Learn. Res.}, vol.~13, no. Apr, pp. 1037--1057,
  2012.

\bibitem{yan2009linear}
X.~Yan and X.~Su, \emph{Linear regression analysis: {T}heory and
  computing}.\hskip 1em plus 0.5em minus 0.4em\relax World {S}cientific, 2009.

\bibitem{candes2009near}
E.~J. Cand{\`e}s, Y.~Plan \emph{et~al.}, ``Near-ideal model selection by
  {$l_1$} minimization,'' \emph{Ann. Stat.}, vol.~37, no.~5A, pp. 2145--2177,
  2009.

\bibitem{candes2011probabilistic}
E.~J. Candes and Y.~Plan, ``A probabilistic and {RIP}less theory of compressed
  sensing,'' \emph{IEEE Trans. Inf. Theory}, vol.~57, no.~11, pp. 7235--7254,
  2011.

\bibitem{candes2006stable}
E.~J. Candes, J.~K. Romberg, and T.~Tao, ``Stable signal recovery from
  incomplete and inaccurate measurements,'' \emph{Comm. Pure Appl. Math.},
  vol.~59, no.~8, pp. 1207--1223, 2006.

\bibitem{cai_l1_minimization}
T.~T. Cai, G.~Xu, and J.~Zhang, ``On recovery of sparse signals via $l_1$
  minimization,'' \emph{IEEE Trans. Inf. Theory}, vol.~55, no.~7, pp.
  3388--3397, July 2009.

\bibitem{cai2010stable}
T.~T. Cai, L.~Wang, and G.~Xu, ``Stable recovery of sparse signals and an
  oracle inequality,'' \emph{IEEE Trans. Inf. Theory}, vol.~56, no.~7, pp.
  3516--3522, 2010.

\bibitem{elad2010sparse}
M.~Elad, \emph{Sparse and redundant representation}.\hskip 1em plus 0.5em minus
  0.4em\relax Springer, 2010.

\end{thebibliography}
\end{document}